\documentclass[letterpaper]{article} 
\usepackage{aaai2026} 

\usepackage{times}  
\usepackage{helvet}  
\usepackage{courier}  
\usepackage[hyphens]{url}  
\usepackage{graphicx} 
\usepackage{natbib}  
\usepackage{caption} 
\usepackage{algorithm}
\usepackage{algorithmic}

\usepackage{newfloat}
\usepackage{listings}
\DeclareCaptionStyle{ruled}{labelfont=normalfont,labelsep=colon,strut=off} 
\floatstyle{ruled}
\newfloat{listing}{tb}{lst}{}
\floatname{listing}{Listing}
\usepackage{amsmath}
\usepackage{amsthm,thm-restate}
\usepackage{amsfonts}
\usepackage{amssymb}
\usepackage{subcaption}
\usepackage{xspace}
\usepackage{xcolor}
\usepackage{enumitem}
\usepackage{cleveref}
\usepackage{tikz}
\usetikzlibrary{positioning}
\usepackage{booktabs}
\usepackage{comment}

\newtheorem{theorem}{Theorem}
\newtheorem{lemma}{Lemma}

\newtheorem{observation}{Observation}

\newcommand{\tool}{{CoVerD}\xspace}
\newboolean{is_extended}
\setboolean{is_extended}{true}

\nocopyright

\title{Tight Robustness Certification Through the Convex Hull of $\ell_0$ Attacks}
\author {
    Yuval Shapira,
    Dana Drachsler-Cohen
}
\affiliations {
    Technion - Israel Institute of Technology\\
    shapirayuval@campus.technion.ac.il, ddana@ee.technion.ac.il
}

\begin{document}

\maketitle

\begin{abstract}
Few-pixel attacks mislead a classifier by modifying a few pixels of an image. Their perturbation space is an $\ell_0$-ball, which is not convex, unlike $\ell_p$-balls for $p\geq 1$. However, existing local robustness verifiers typically scale by relying on linear bound propagation, which captures convex perturbation spaces. We show that the convex hull of an $\ell_0$-ball is the intersection of its bounding box and an asymmetrically scaled $\ell_1$-like polytope. The volumes of the convex hull and this polytope are nearly equal as the input dimension increases. We then show a linear bound propagation that precisely computes bounds over the convex hull and is significantly tighter than bound propagations over the bounding box or our $\ell_1$-like polytope. This bound propagation scales the state-of-the-art  $\ell_0$ verifier on its most challenging robustness benchmarks by 1.24x-7.07x, with a geometric mean of 3.16.
\end{abstract}

\begin{links}
    \link{Code}{https://github.com/YuvShap/top-t-CoVerD}
\end{links}
\section{Introduction}
Image network classifiers are a central part of many safety-critical systems in healthcare~\cite{Miller2023SelfSupervisedMedical}, autonomous driving~\cite{Waymo2025Verge}, and automated visual inspection~\cite{AmazonUVeye2023}. 
However, neural networks are susceptible to adversarial example attacks~\cite{INTRIGUING_PROP}.
An adversarial attack injects a small perturbation, often confined to the $\ell_p$-ball of a correctly classified input, to mislead the classifier. Local robustness~\cite{Bastani2016Measuring} is the main safety property for showing the resilience of networks to adversarial attacks. 

Many robustness verifiers analyze the local robustness of neural networks to perturbations in $\ell_\infty$-balls~\cite{TjengXT19,ZhangWCHD18,MullerS0PV21,GehrMDTCV18,KatzBDJK17}, $\ell_2$-balls~\cite{Leino21,Huang21}, and $\ell_1$-balls~\cite{NEURIPS2021_26657d5f,WuZ21}.
To scale, most incomplete verifiers rely on linear bound propagation that overapproximates the network's computations with convex polytopes.  
Since $\ell_p$-balls are convex for $p\geq 1$, such analysis does not introduce overapproximation error in the perturbation space.
However, $\ell_0$-balls are not convex. This raises the question: can linear bound propagation avoid introducing overapproximation error in the non-convex $\ell_0$ perturbation space?  

We mathematically characterize the convex hull of an $\ell_0$-ball around an input $x$ (single- or multi-channel).
We show that the convex hull is the intersection of the bounding box of the $\ell_0$-ball and an asymmetrically scaled $\ell_1$-like polytope.  
We show that the relative excess volume of this polytope 
compared to the convex hull converges exponentially to zero, suggesting it may be a good overapproximation for linear bound propagation. 
 
We then present a linear bound propagation that precisely computes the minimum and maximum of a linear function over an $\ell_0$-ball, which coincide with those over its convex hull.
We show that the minimum and maximum of a linear function $f$ over an $\ell_0$-ball of radius $t$ depend on the sum of the $t$ lowest (for the minimum) or highest (for the maximum) input entry contributions to~$f$.  
This bound propagation generalizes prior bound propagations for few-pixel attacks~\cite{Chiang2020Certified,NEURIPS2020_0cbc5671} to any box input domain and to multi-channel inputs. 
We also present a bound propagation for our $\ell_1$-like polytope and show that the minimum and maximum depend on the product of $t$ and the lowest or highest input entry contribution to $f$.  
That is, this polytope is a looser overapproximation than the convex hull, although  their volumes are very close. 

We integrate our bound propagation in GPUPoly~\cite{MullerS0PV21}, which is repeatedly called by~\tool~\cite{ShapiraWSD24}, the state-of-the-art complete (exact) $\ell_0$ robustness verifier.
We evaluate it over fully-connected and convolutional classifiers, for MNIST, Fashion-MNIST, and CIFAR-10.
Our bound propagation boosts~\tool on its most challenging robustness benchmarks by 1.24x-7.07x, with a geometric mean of 3.16.
 
In summary, our main contributions are:
\begin{itemize}[nosep,nolistsep]
   \item A characterization of the convex hull of $\ell_0$ perturbations.
   \item A linear bound propagation that precisely computes the minimum and maximum of a linear function over an $\ell_0$-ball, which is significantly tighter than bound propagations over its bounding box or our $\ell_1$-like polytope.
   \item An integration of our bound propagation in GPUPoly, for boosting $\ell_0$ robustness verification. 
\end{itemize} 
\section{Preliminaries}
In this section, we provide our notation and setting.

An input $x$ is in a bounded box domain $\mathcal{D}_v=\prod_{i=1}^v[a_i,b_i]$ if $x$ is 
single-channel, or $\mathcal{D}_v=\prod_{i=1}^v(\prod_{j=1}^d[a_i^{(j)},b_i^{(j)}])$ if $x$ is multi-channel.
For example, an RGB image $x$ is in $\mathcal{D}_v=([0,1]^3)^v$. 
We denote an entry by $x_i$ for $i\in [v]=\{1,\ldots,v\}$ and a channel by $x^{(j)}_i$ for $i\in [v]$ and $j\in [d]$.

We focus on the local robustness of a network classifier to few-pixel attacks~\cite{CroceASF022}.  
A classifier maps an input to a score vector over $c$ labels: $N:\mathbb{R}^v\to \mathbb{R}^c$. 
The classification of an input is the label with the maximal score: $\text{class}_N(x)=\text{argmax}(N(x))$.
An adversarial attack considers an input $\bar{x}\in \mathcal{D}_v$ and computes a small perturbation $r$ such that $\bar{x}+r\in \mathcal{D}_v$ and the perturbed input is classified differently:
$\text{class}_N(\bar{x})\neq  \text{class}_N(\bar{x}+r)$. In few-pixel attacks, the attacker is given a bound $t\in [v]$ on the number of entries that can be perturbed, which is often very small, and computes $r$ whose $\ell_0$ norm is at most $t$:
$|\{i\in[v]\mid r_i\neq 0\}|\leq t$.

We consider a generalized setting in which the attacker is limited to perturbing a subset of pixels, $\mathcal{K}\subseteq [v]$.
 Note that $\mathcal{K}$ can be arbitrarily large, and in particular it can be~$[v]$, as in our experiments.
Since the attacker cannot perturb the pixels in $[v]\setminus \mathcal{K}$, given $\bar{x}\in \mathcal{D}_v$, 
the effective input space is $\mathcal{D}=\prod_{i\in \mathcal{K}}[a_i,b_i]$ (and similarly for multi-channel inputs), i.e., pixels not in $\mathcal{K}$ are treated as constants. Formally, the extension of $y\in \mathcal{D}$ to $y'\in \mathcal{D}_v$ sets $y'_i=y_i$, for $i\in \mathcal{K}$, and $y'_i=\bar{x}_i$, for $i\in [v]\setminus \mathcal{K}$. We abuse notation: for $\bar{x}\in \mathcal{D}_v$, we also write $\bar{x}$ for its projection onto $\mathcal{K}$, and for $y\in \mathcal{D}$, we also write $y$ for its extension to $\mathcal{D}_v$.
The space of all perturbed inputs is the {\bf $\ell_0$-ball of $\bar{x}$ with radius~$t$}:
\begin{align*}
   \mathcal{B}_0^t(\bar{x})=\{y\in\mathcal{D}\mid|\{i\in\mathcal{K}\mid y_i\neq \bar{x}_i\}|\leq t\}
\end{align*}
A network $N$ is locally robust in $\mathcal{B}_0^t(\bar{x})$ if it classifies all inputs the same: $\forall y\in \mathcal{B}_0^t(\bar{x}).\ \text{class}_N(\bar{x})=  \text{class}_N(y)$.
Without loss of generality and to simplify notations,
we assume $\mathcal{K}=[k]$, for $ k\in [t, v]$ (though our formulation, theorems, and bound propagations hold for any $\mathcal{K}\subseteq [v]$). In particular, we write $\mathcal{D}=\prod_{i=1}^k[a_i,b_i]$.
\Cref{fig:all}~(left) shows the $\ell_0$-balls of three single-channel inputs, for $t=2$. 
The figure shows that the $\ell_0$-ball is the union of $\binom{3}{2}$ planes. Generally, an $\ell_0$-ball is a union of $\binom{k}{t}$ $t\cdot d$-dimensional flats. 

Although our focus is on robustness of image classifiers, our formulation, theorems, and bound propagations apply to other classifiers.
For example, to text classifiers, which are susceptible to word replacement attacks~\cite{AlzantotSEHSC18}.
Given an input sentence, the attacker replaces up to $t$ words. 
Technically, if the words are independently embedded (e.g., with Word2Vec~\cite{NIPS2013_word2vec}), the problem setting is similar:
given an input sentence $\bar{x}$ with $v$ words $\bar{x}_1,\dots,\bar{x}_v$, the input domain is $\mathcal{D}_v=\prod_{i=1}^v(\prod_{j=1}^d[a_i^{(j)},b_i^{(j)}])$, where $\prod_{j=1}^d[a_i^{(j)},b_i^{(j)}]$ is the bounding box of the embedding of $\bar{x}_i$ and its replacements. 
\section{The Convex Hull of $\ell_0$-Balls}
In this section, we characterize the convex hull of an $\ell_0$-ball.

\begin{figure*}[t]
\centering
\begin{subfigure}{0.3\textwidth}
    \includegraphics[width=\textwidth]{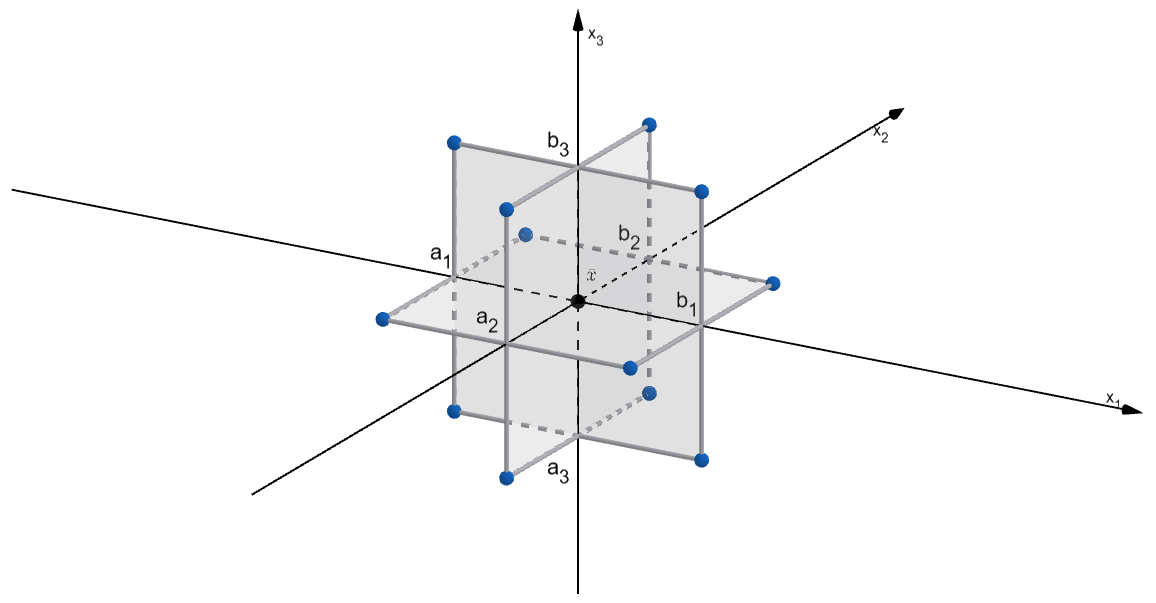} 
    \caption{The $\ell_0$-ball of $\bar{x}=(0,0,0)$.}
    \label{fig:pertubation_space}
\end{subfigure}\hfill
\begin{subfigure}{0.3\textwidth}
    \includegraphics[width=\textwidth]{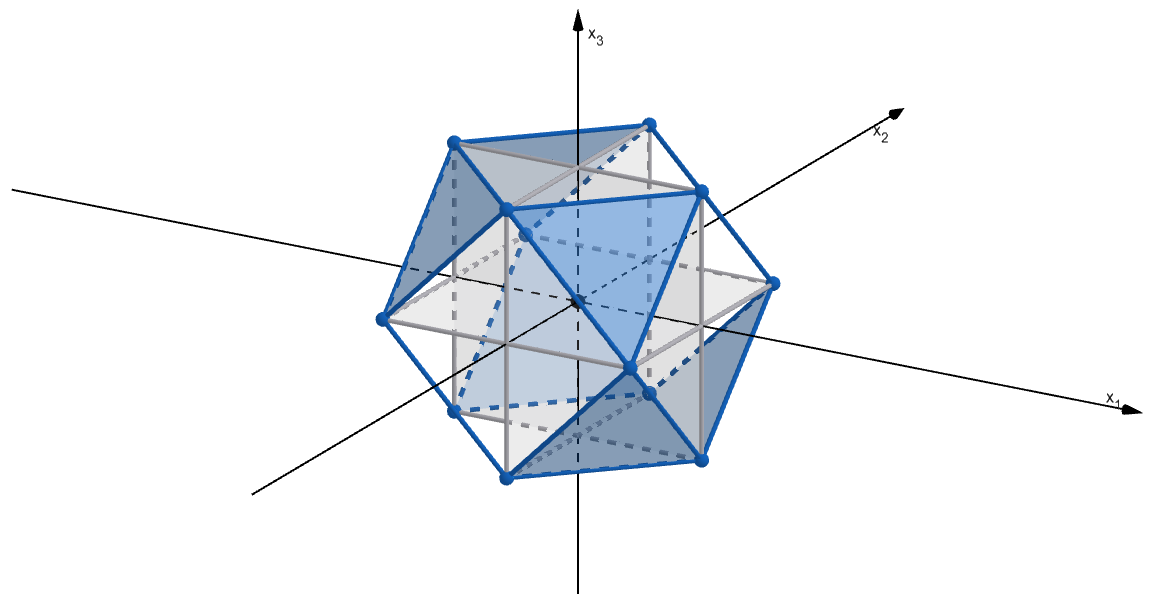} 
    \caption{The convex hull of the $\ell_0$-ball.}
    \label{fig:pertubation_space_triangles} 
\end{subfigure}\hfill
\begin{subfigure}{0.3\textwidth}
    \includegraphics[width=\textwidth]
    {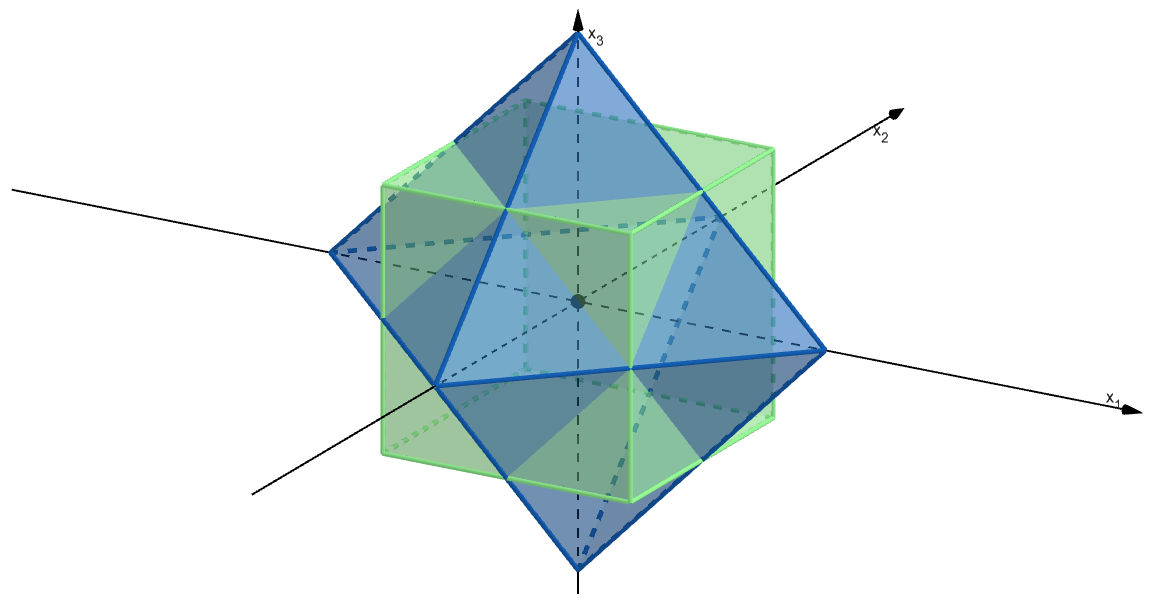}
    \caption{The intersection of $\mathcal{D}$ \& the $\ell_1$-ball.}
    \label{fig:intersection}
\end{subfigure}
\begin{subfigure}{0.3\textwidth}
    \includegraphics[width=\textwidth]{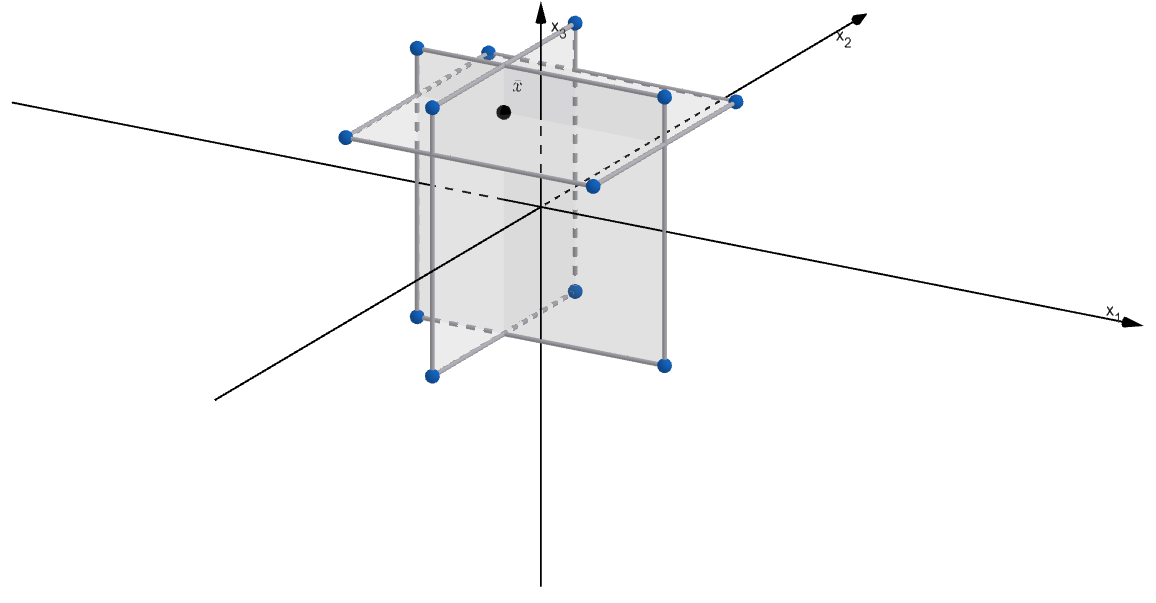} 
    \caption{The $\ell_0$-ball of $\bar{x}=(-0.3,0,0.65)$.}
    \label{fig:pertubation_space2}
\end{subfigure}\hfill
\begin{subfigure}{0.3\textwidth}
    \includegraphics[width=\textwidth]{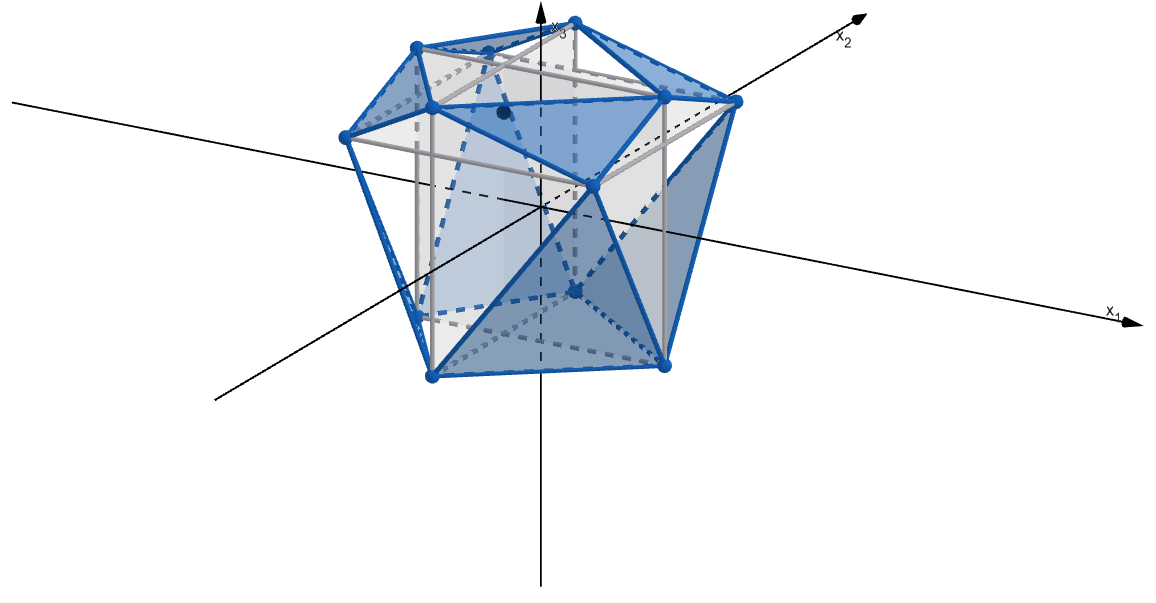} 
    \caption{The convex hull of the $\ell_0$-ball.}
    \label{fig:pertubation_space_triangles2} 
\end{subfigure}\hfill
\begin{subfigure}{0.3\textwidth}
    \includegraphics[width=\textwidth]
    {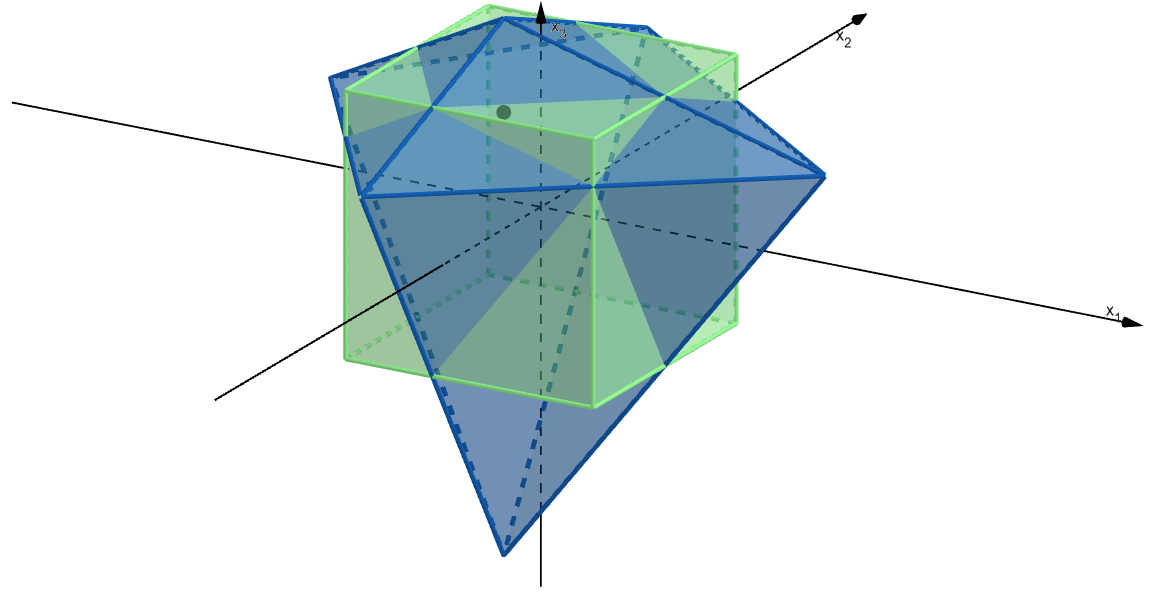}
    \caption{The intersection of $\mathcal{D}$ \& $ \widetilde{\mathcal{B}}_1^t(\bar{x})$.}
    \label{fig:intersection2}
\end{subfigure}
\begin{subfigure}{0.3\textwidth}
    \includegraphics[width=\textwidth]{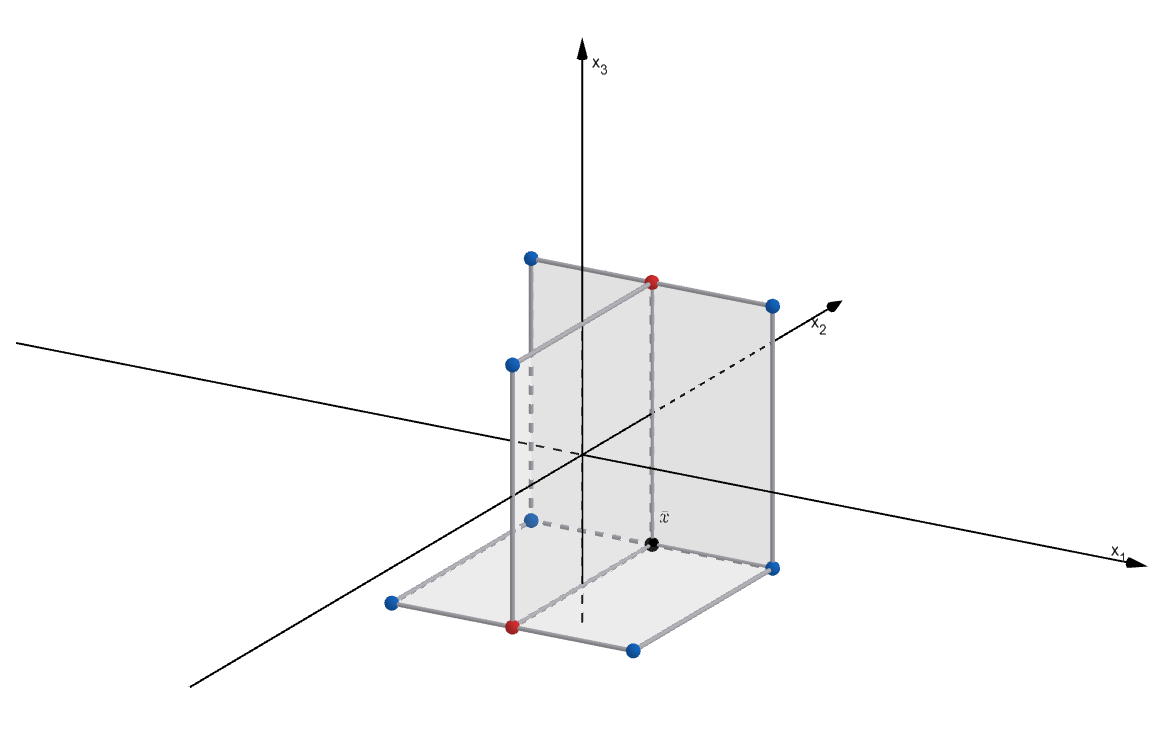} 
    \caption{The $\ell_0$-ball of $\bar{x}=(0,1,-1)$.}
    \label{fig:pertubation_space3}
\end{subfigure}\hfill
\begin{subfigure}{0.3\textwidth}
    \includegraphics[width=\textwidth]{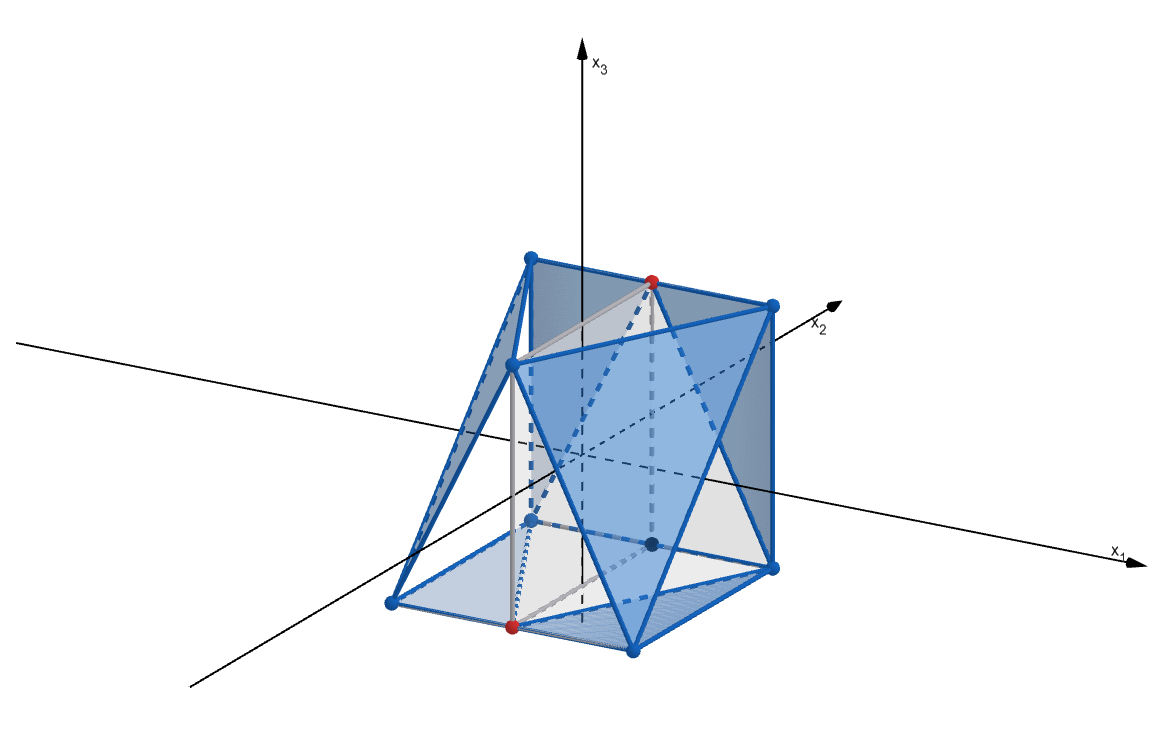} 
    \caption{The convex hull of the $\ell_0$-ball.}
    \label{fig:pertubation_space_triangles3} 
\end{subfigure}\hfill
\begin{subfigure}{0.3\textwidth}
    \includegraphics[width=\textwidth]
    {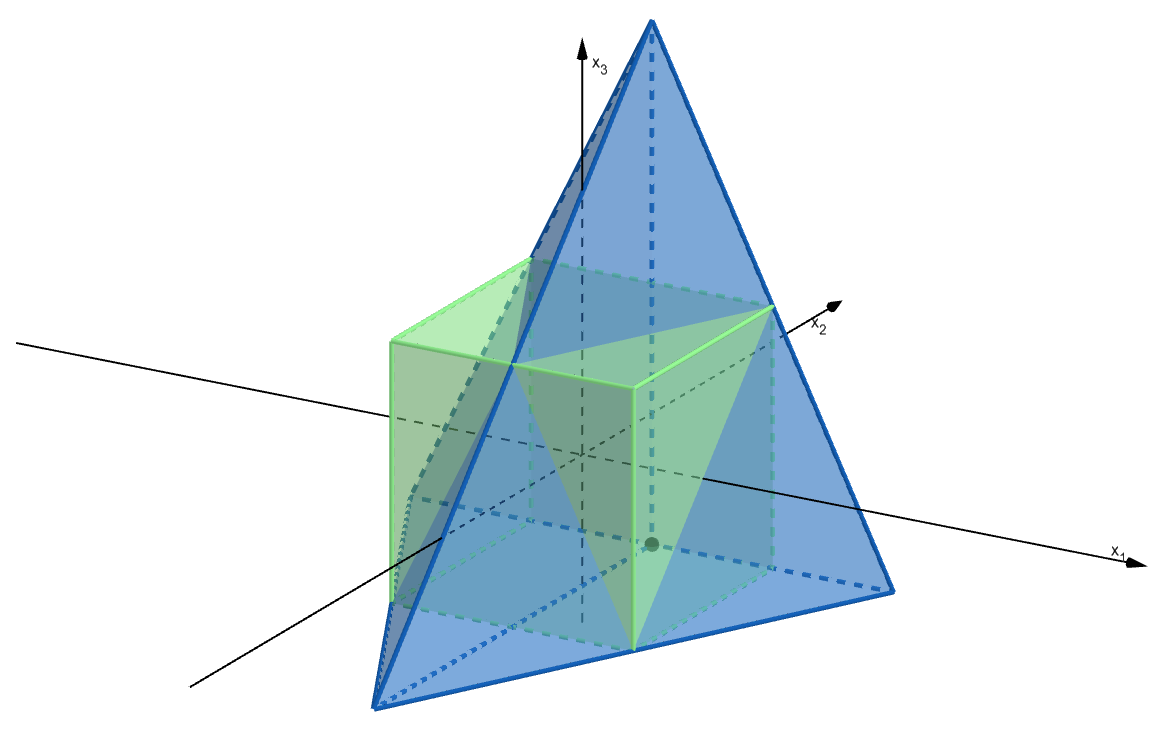}
    \caption{The intersection of $\mathcal{D}$ \& $ \widetilde{\mathcal{B}}_1^t(\bar{x})$.}
    \label{fig:intersection3}
\end{subfigure}
\caption{Illustration of the perturbation space, for $\mathcal{D}=\prod_{i=1}^k[a_i,b_i]=[-1,1]^3$ and $t=2$. Left: The $\ell_0$-ball of an input~$\bar{x}$, $\mathcal{B}_0^t(\bar{x})$. Middle: Its convex hull. Right: The convex hull as the intersection of $\mathcal{D}$ and an $\ell_1$-like polytope (\Cref{thm:single_channel_conv_hull}), $\widetilde{\mathcal{B}}_1^t(\bar{x})=\{y\in\mathbb{R}^k \mid \sum_{i=1}^k\delta_{\bar{x}}^i(y)\leq t\}$ (see~\Cref{eq:delta}).
The plots show $\bar{x}$ as a black dot and the corners of the $\ell_0$-ball as blue dots.
}\label{fig:all}
\end{figure*} 

\subsection{The Convex Hull for Single-Channel Inputs}
In this section, we focus on a single-channel input $\bar{x}$.

\paragraph{Illustration}
We illustrate by considering the $\ell_0$-ball of $\bar{x}=(0,0,0)$ that lies in the center of $ \mathcal{D}=[-1,1]^3$, shown in~\Cref{fig:pertubation_space}. 
The figure illustrates that the $\ell_0$-ball is not convex, unlike $\ell_p$-balls for $p\geq 1$. 
The corners $E_{\mathcal{B}_0^t(\bar{x})}$ of the $\ell_0$-ball are all inputs $y\in \mathbb{R}^k$ such that $k-t$ of their entries are zeros, and the other entries $y_i$ are one of the bounds $-1$ or $1$ (blue dots in~\Cref{fig:pertubation_space}).
The convex hull of the $\ell_0$-ball is the smallest convex set that contains it. It is obtained by connecting the corners $E_{\mathcal{B}_0^{t}(\bar{x})}$ (\Cref{fig:pertubation_space_triangles}).
We shortly provide a characterization of this convex hull.  
One may wonder whether it is crucial, or whether it suffices to overapproximate the $\ell_0$-ball by a containing convex set that can be analyzed by existing robustness verifiers.
For example, most verifiers are designed for box neighborhoods. Thus, one may attempt verifying local robustness in an $\ell_0$-ball by overapproximating it with its bounding box, whose boundary contains the corners of the $\ell_0$-ball, and checking robustness in this box. However, this box equals the input space~$\mathcal{D}$. For large $k$, the network is not locally robust in $\mathcal{D}$ (e.g., for $k=v$, it is not robust since the network does not classify all inputs the same). Thus, this approach usually fails for large $k$.
Instead, one may overapproximate the $\ell_0$-ball with the tightest convex $\ell_p$-ball, which is the $\ell_1$-ball: $\{y \in \mathbb{R}^k \mid \Vert y\Vert_1\leq t\}$~\cite{NIPS2006_75b9b6dc, pmlr-v235-baninajjar24a}, and submit it to an $\ell_1$-ball robustness verifier~\cite{NEURIPS2021_26657d5f}. 
However, the $\ell_1$-ball has sharp corners whose $\ell_\infty$ distance from $\mathcal{D}$ can be very high (\Cref{fig:intersection}). 
This overapproximation error can fail the verification. Our first theorem shows that, in the case where $\bar{x}$ is in the center of $\mathcal{D}$, the convex hull of the $\ell_0$-ball of $\bar{x}$ is the intersection of $\mathcal{D}$ and the $\ell_1$-ball. \Cref{fig:intersection} shows  $\mathcal{D}$ (in green) and the $\ell_1$-ball with radius $t=2$ (in blue).

\paragraph{General case}
We next present our characterization for the general case where $\bar{x}$ may not be in the center of $\mathcal{D}$, e.g., the $\bar{x}$ in~\Cref{fig:pertubation_space2,fig:pertubation_space3} (the black dot). 
In this case, the corners $E_{\mathcal{B}_0^t(\bar{x})}$ (the blue dots) are inputs $y\in \mathbb{R}^k$ such that $k-t$ of their entries $y_i$ are $\bar{x}_i$, and the other entries $y_i$ are one of the bounds $a_i$ or $b_i$ (which may be equal to $\bar{x}_i$, e.g.,~\Cref{fig:pertubation_space3}).
We note that, unlike the simpler case, it can happen that an input $y$ is not a corner even if it has $k-t$ entries $y_i=\bar{x}_i$ and its other entries $y_i$ are $a_i$ or $b_i$ (red dots and $\bar{x}$ itself in~\Cref{fig:pertubation_space3}).
The convex hull is obtained by connecting the corners $E_{\mathcal{B}_0^t(\bar{x})}$ (\Cref{fig:pertubation_space_triangles2,fig:pertubation_space_triangles3}).
We show that it equals the intersection of $\mathcal{D}$ and an \emph{asymmetrically scaled} $\ell_1$-ball around $\bar{x}$.
To define this polytope, we first define the \emph{asymmetrically scaled $i$ distance}, for $i\in [k]$. 
It measures how far $y_i$ is from $\bar{x}_i$: if $y_i=\bar{x}_i$, this distance is zero, otherwise if $y_i$ is one of the bounds $a_i$ or $b_i$, this distance is one, and otherwise this distance is proportional to the distance between ${\bar{x}}_{i}$ and the bound that is closer to $y_{i}$ (than to ${\bar{x}}_{i}$).
Formally, the asymmetrically scaled $i$ distance from input $y$ to $\bar{x}$ is the difference of their $i^{\text{th}}$ entries normalized by the distance of $\bar{x}_i$ and the bound $b_i$, if $y_i > \bar{x}_i$, or $a_i$, if $y_i < \bar{x}_i$:
\begin{equation} \label{eq:delta}
    \delta_{\bar{x},a_i,b_i}^{i}(y)=\frac{y_i-\bar{x}_i}{\mathbf{1}_{\{y_i>\bar{x}_i\}}(b_i-\bar{x}_i)+\mathbf{1}_{\{y_i<\bar{x}_i\}}(a_i-\bar{x}_i)}
\end{equation}
We omit $a_i$ and $b_i$ to simplify notation.
We define $\delta_{\bar{x}}^i(y)=0$, if $y_i=\bar{x}_i$ (when the numerator and denominator are $0$) and $\delta_{\bar{x}}^i(y)=\infty$, if $\bar{x}_i=b_i$ and $y_i>\bar{x}_i$ or if $\bar{x}_i=a_i$ and $y_i<\bar{x}_i$ (when the denominator is $0$).
Thus, $\delta_{\bar{x}}^{i}(y)\in [0,\infty]$. 
If $y\in \mathcal{D}$, then $\delta_{\bar{x}}^{i}(y)\in[0,1]$, since for $y_i=\bar{x}_i$,  $\delta_{\bar{x}}^{i}(y)=0$, and for $y_i\in\{a_i,b_i\}\setminus\{\bar{x}_i\}$,  $\delta_{\bar{x}}^{i}(y)=1$.
As an example, consider $[a_i,b_i]=[-1,1]$ and $y_i=0.3$. For $\bar{x}_i=0$ (the center),  $\delta_{\bar{x}}^{i}(y)=0.3$. For $\bar{x}_i=0.7$,  
$\delta_{\bar{x}}^{i}(y)=\frac{0.3-0.7}{0+1\cdot (-1-0.7)}=0.235$.

The \emph{asymmetrically scaled $\ell_1$-ball} $\widetilde{\mathcal{B}}_1^t(\bar{x})$ is the set of inputs whose sum over all asymmetrically scaled $i$ distances from $\bar{x}$ (\Cref{eq:delta})   
is at most $t$:
 $$\widetilde{\mathcal{B}}_1^t(\bar{x})=\{y\in\mathbb{R}^k \mid \sum_{i=1}^k\delta_{\bar{x}}^i(y)\leq t\}$$
 Note that $\widetilde{\mathcal{B}}_1^t(\bar{x})$ is defined over $y\in\mathbb{R}^k$ (unlike $\mathcal{B}_0^t(\bar{x})$, defined over $y\in \mathcal{D}$), because it generalizes the $\ell_1$-ball, defined over $y\in \mathbb{R}^k$.
\Cref{fig:intersection2,fig:intersection3} exemplify $\widetilde{\mathcal{B}}_1^t(\bar{x})$ (in blue). 

\paragraph{Theorem} We next state and prove our main theorem, which characterizes the convex hull of an $\ell_0$-ball. In the following, we denote by $Conv(S)$ the convex hull of a set $S$. 
\begin{theorem}\label{thm:single_channel_conv_hull} 
$Conv(\mathcal{B}_0^t(\bar{x}))=\mathcal{D}\cap \widetilde{\mathcal{B}}_1^t(\bar{x})$.
\end{theorem}

\begin{proof}
\phantom{a}\\
$\bullet$ $Conv(\mathcal{B}_0^t(\bar{x}))\subseteq \mathcal{D}\cap \widetilde{\mathcal{B}}_1^t(\bar{x})$: 
  By definition, 
  $\mathcal{B}_0^t(\bar{x})\subseteq \mathcal{D}$ and $\mathcal{B}_0^t(\bar{x})\subseteq\widetilde{\mathcal{B}}_1^t(\bar{x})$. Hence, $\mathcal{B}_0^t(\bar{x})\subseteq \mathcal{D} \cap \widetilde{\mathcal{B}}_1^t(\bar{x})$. Since $Conv(\cdot)$ is monotone with respect to set inclusion,
    $Conv(\mathcal{B}_0^t(\bar{x}))\subseteq Conv(\mathcal{D}\cap \widetilde{\mathcal{B}}_1^t(\bar{x}))$.
  The set $\mathcal{D}$ is convex (as a box) and the set $\widetilde{\mathcal{B}}_1^t(\bar{x})$ is convex (by Lemma 1, appendix\ifthenelse{\boolean{is_extended}}{}{, in the extended version}) and thus so is $\mathcal{D}\cap \widetilde{\mathcal{B}}_1^t(\bar{x})$. Thus, $Conv(\mathcal{B}_0^t(\bar{x}))\subseteq \mathcal{D}\cap \widetilde{\mathcal{B}}_1^t(\bar{x})$.
  \\   
 $\bullet$   $ \mathcal{D}\cap\widetilde{\mathcal{B}}_1^t(\bar{x})\subseteq Conv(\mathcal{B}_0^t(\bar{x}))$: 
  The sets $\mathcal{D}$ and $\widetilde{\mathcal{B}}_1^t(\bar{x})$ are convex and compact (as a box or by Lemma 1, appendix), and thus so is
    $\mathcal{D}\cap\widetilde{\mathcal{B}}_1^t(\bar{x})$.
     By the Krein-Milman Theorem~\cite{rockafellar1997convex}, a compact convex set is the convex hull of its extreme points. Thus, $\mathcal{D}\cap\widetilde{\mathcal{B}}_1^t(\bar{x})=Conv(E_\cap)$, where $E_\cap$ is the set of extreme points of $\mathcal{D}\cap\widetilde{\mathcal{B}}_1^t(\bar{x})$. 
    Since $E_\cap \subseteq \mathcal{B}_0^t(\bar{x})$ (Lemma 2, appendix),
     $Conv(E_\cap) \subseteq Conv(\mathcal{B}_0^t(\bar{x}))$.
\end{proof}  

\paragraph{Volumes} We next study the volumes of the convex hull, $\mathcal{D}\cap \widetilde{\mathcal{B}}_1^t(\bar{x})$, and $\widetilde{\mathcal{B}}_1^t(\bar{x})$ (the volume of $\mathcal{D}$ is $\prod_{i=1}^k(b_i-a_i)$).
We show that the relative excess volume of $\widetilde{\mathcal{B}}_1^t(\bar{x})$ compared to $\mathcal{D}\cap \widetilde{\mathcal{B}}_1^t(\bar{x})$ converges exponentially to $0$ as $k$ increases.
We begin by providing closed-form expressions for these volumes, which depend only on $\mathcal{D}$ and~$t$ (not on the input $\bar{x}$).
 In particular, this implies that the volumes of the convex hulls shown in \Cref{fig:pertubation_space_triangles,fig:pertubation_space_triangles2,fig:pertubation_space_triangles3} are equal.
\begin{restatable}[]{theorem}{ftb}\label{thm2}
    \begin{enumerate}[label=\arabic*.]
        \item $\text{vol}(\widetilde{\mathcal{B}}_1^t(\bar{x}))=\text{vol}(\mathcal{D})\frac{t^k}{k!}$.
        \item $\text{vol}(\mathcal{D}\cap\widetilde{\mathcal{B}}_1^t(\bar{x}))=\text{vol}(\mathcal{D})\frac{t^k}{k!}\sum_{r=0}^{t-1}(-1)^r\binom{k}{r}(1-\frac{r}{t})^k$.
    \end{enumerate}
\end{restatable}
\begin{proof}[Proof outline (proof is in the appendix)]
We partition $\widetilde{\mathcal{B}}_1^t(\bar{x})$ and $\mathcal{D}\cap \widetilde{\mathcal{B}}_1^t(\bar{x})$ into $2^k$ parts, one for each orthant. 
Each part is a scaled version of $\Delta_{k,t}=\{z\in[0,\infty)^k\mid\sum_{i=1}^kz_i\leq t\}$ for $\widetilde{\mathcal{B}}_1^t(\bar{x})$ and  $[0,1]^k\cap \Delta_{k,t}$ for $\mathcal{D}\cap\widetilde{\mathcal{B}}_1^t(\bar{x})$.
It is known that $\text{vol}(\Delta_{k,t})=\frac{t^k}{k!}$.
Thus, summing over the scaled volumes yields $\text{vol}(\widetilde{\mathcal{B}}_1^t(\bar{x}))=\text{vol}(\mathcal{D})\frac{t^k}{k!}$ (proving bullet 1)
and $\text{vol}(\mathcal{D}\cap \widetilde{\mathcal{B}}_1^t(\bar{x}))=\text{vol}(\mathcal{D})\cdot \text{vol}(\Delta_{k,t}\cap[0,1]^k)$.
 We show that $\text{vol}(\Delta_{k,t}\cap[0,1]^k)$ is equal to the cumulative distribution function (CDF) of the Irwin–Hall distribution \cite{Irwin, Hall} evaluated at $t$, which is $\frac{t^k}{k!}\sum_{r=0}^{t-1}(-1)^r\binom{k}{r}(1-\frac{r}{t})^k$ (proving bullet 2).
 Technically, to compute the volume, we compute the probability that a uniformly sampled input from $[0,1]^k$ lies in $\Delta_{k,t}$, i.e., that the sum of its entries is at most $t$, and multiply it by $\text{vol}([0,1]^k)=1$. 
This probability is obtained by evaluating the Irwin-Hall CDF at $t$.
\end{proof}

\paragraph{Excess volumes}
We next compare the volume of the convex hull to $\mathcal{D}$ and to $\widetilde{\mathcal{B}}_1^t(\bar{x})$. To shorten notation, we 
write
$\text{vol}(\mathcal{D}\cap\widetilde{\mathcal{B}}_1^t(\bar{x}))=\text{vol}(\widetilde{\mathcal{B}}_1^t(\bar{x}))\sum_{r=0}^{t-1}c_{r,k,t}$, where $c_{r,k,t}=(-1)^r\binom{k}{r}(1-\frac{r}{t})^k$, for $r\in\{0,\dots,t-1\}$.  
Note that $c_{r,k,t}=1$, for $r=0$, and $c_{r,k,t}$ converges exponentially to $0$ as $k$ increases, for $r>0$. 
The relative excess volume of $\widetilde{\mathcal{B}}_1^t(\bar{x})$:
\begin{align*}
 \frac{\text{vol}(\widetilde{\mathcal{B}}_1^t(\bar{x})\setminus(\mathcal{D}\cap\widetilde{\mathcal{B}}_1^t(\bar{x}))))}{\text{vol}(\mathcal{D}\cap\widetilde{\mathcal{B}}_1^t(\bar{x}))}
=\frac{-\sum_{r=1}^{t-1}c_{r,k,t}}{1+\sum_{r=1}^{t-1}c_{r,k,t}} 
\end{align*}
This expression is asymptotically negligible because the numerator approaches $0$ and the denominator approaches $1$. The relative excess volume of $\mathcal{D}$ is:
\begin{align*}
 \frac{\text{vol}(\mathcal{D}\setminus(\mathcal{D}\cap\widetilde{\mathcal{B}}_1^t(\bar{x}))))}{\text{vol}(\mathcal{D}\cap\widetilde{\mathcal{B}}_1^t(\bar{x}))}=
\frac{\frac{k!}{t^k}-\sum_{r=0}^{t-1}c_{r,k,t}}{1+\sum_{r=1}^{t-1}c_{r,k,t}}
\end{align*}
This expression goes to infinity, since $\frac{k!}{t^k}$ diverges with $k$.
\Cref{fig:volumes} illustrates this behavior for different values of $t$. 
It shows that for very small values of $k$, the relative excess volume of $\mathcal{D}$ is very small, but for $k\geq 20$ it is very large (as context, for MNIST images $k=784$). In contrast, the relative excess volume of $\widetilde{\mathcal{B}}_1^t(\bar{x})$ converges to zero very quickly.

\subsection{The Convex Hull for Multi-Channel Inputs}
In this section, we focus on a multi-channel input $\bar{x}$.
In this setting, the attacker chooses $t$ entries $\bar{x}_i$ to perturb across all their channels (overall, $t\cdot d$ values).    
To characterize the convex hull, we extend our definitions to multi-channel inputs.
The \emph{asymmetrically scaled $(i,j)$ distance} from $y$ to $\bar{x}$ is:
\begin{align*}
\delta_{\bar{x}}^{i,j}(y)=\frac{y_i^{(j)}-\bar{x}_i^{(j)}}{\mathbf{1}_{y_i^{(j)}>\bar{x}_i^{(j)}}(b_i^{(j)}-\bar{x}_i^{(j)})+\mathbf{1}_{y_i^{(j)}<\bar{x}_i^{(j)}}(a_i^{(j)}-\bar{x}_i^{(j)})}
\end{align*}
This definition is identical to $\delta_{\bar{x}}^{i}(y)$, but over a channel $y_i^{(j)}$, for $i\in [k],j\in[d]$. We extend the asymmetrically scaled $\ell_1$-ball to the set of inputs whose sum, over all entries $i$, of
 the \emph{maximal} asymmetrically scaled $(i,j)$ distances across all channels $j$ is at most $t$. 
Since the sum is over maximal values, we denote it $\widetilde{\mathcal{B}}_{1, \infty}^t(\bar{x})$: $$\widetilde{\mathcal{B}}_{1, \infty}^t(\bar{x})=\{y\in(\mathbb{R}^{d})^{k} \mid \sum_{i=1}^k \max_{j\in[d]}\delta_{\bar{x}}^{i,j}(y)\leq t\}$$
 This definition is identical to $\widetilde{\mathcal{B}}_{1}^t(\bar{x})$, except that it is over the maximal $\delta_{{\bar{x}}}^{i,j}(y)$.
 Like~\Cref{thm:single_channel_conv_hull}, we show that the convex hull of $\mathcal{B}_0^t(\bar{x})$ is the intersection of $\mathcal{D}$ and  $ \widetilde{\mathcal{B}}_{1,\infty}^t(\bar{x})$:
\begin{theorem}
$Conv(\mathcal{B}_0^t(\bar{x}))=\mathcal{D}\cap \widetilde{\mathcal{B}}_{1,\infty}^t(\bar{x})$
\end{theorem}
The proof is identical to the proof of~\Cref{thm:single_channel_conv_hull}, except that it relies on Lemmas 3 and 4 (provided in the appendix), which extend Lemmas 1 and 2 to multi-channel inputs.
The appendix also shows Theorem 4, which extends \Cref{thm2} to multi-channel inputs and implies a a similar trend for the relative excess volumes.

\begin{figure}[t]
\centering
\includegraphics[width=0.95\columnwidth]{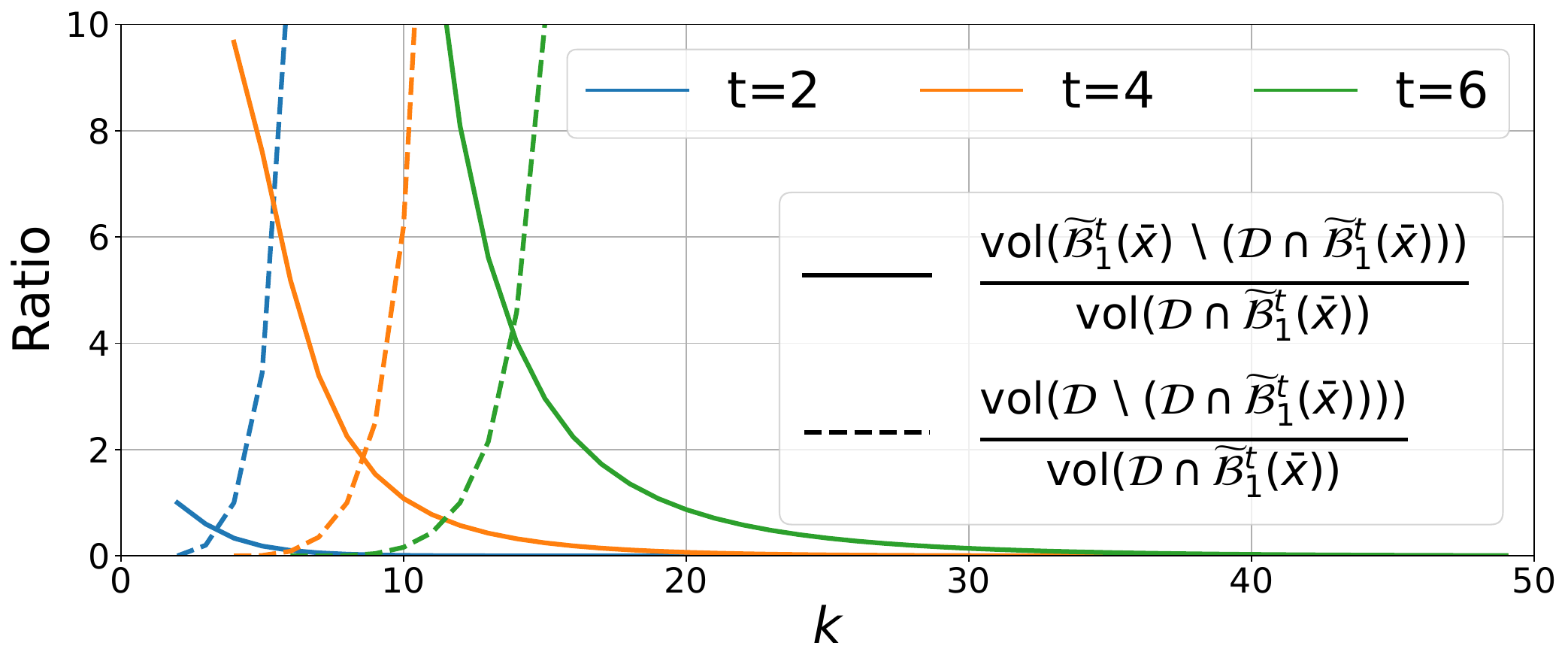}
\caption{The relative excess volumes.}
\label{fig:volumes}
\end{figure} 
\section{Bound Propagations}\label{sec:bound}
In this section, we introduce linear bound propagation methods for local robustness verification over an $\ell_0$-ball.
We begin with background, including bound propagation over a box domain. 
Then, we present an equivalent formulation for the box domain, which we build on in our subsequent bound propagations.
Next, we present bound propagations over an $\ell_0$-ball and over an asymmetrically scaled $\ell_1$-ball.
Lastly, we extend them to multi-channel inputs and compare them. 

\subsection{Background}
Bound propagation is a central component in many neural network verifiers (e.g.,~\citet{ZhangWCHD18,MullerS0PV21,SinghGPV19}).
This technique bounds every neuron $n$ within a real-valued interval $[l,u]\in\mathbb{R}^2$. 
These bounds allow the verifier to efficiently check properties (i.e., constraints over the network's inputs and outputs) as well as to tighten linear relaxations of the network's computations, which reduce the verification time.

To obtain a lower bound $l$ for a neuron $n$, a verifier computes an affine function that lower bounds the computation of $n$ and computes its minimum.
To shorten notations, we assume that the affine function is linear. 
Formally, for single-channel inputs, a verifier computes a linear function over the network's inputs $y\mapsto\sum_{i=1}^kw_iy_i$, for  $w_i\in \mathbb{R}$, lower bounding $n$ in the property's input subspace $\mathcal{I}\subseteq \mathbb{R}^k$: $\forall y\in \mathcal{I}.\ n(y)\geq \sum_{i=1}^kw_iy_i$. It then computes the minimum:  
$l=\min_{y\in \mathcal{I}}{\sum_{i=1}^kw_iy_i}$, based on the shape of $\mathcal{I}$.
Commonly, $\mathcal{I}$ is a box domain (shortly described). 
Similarly, an upper bound $u$ of $n$ is obtained by computing a linear function that upper bounds the computation of $n$ and  computing its maximum. 
The linear functions are computed by linear relaxation to the nonlinear activation functions (e.g., ReLU).
Our bound propagation is orthogonal to the linear relaxation computation. In particular, it is applicable to relaxations computed by a backward pass (e.g.,~\citet{ZhangWCHD18,SinghGPV19}), a forward pass (e.g.,~\citet{NEURIPS2018_2ecd2bd9,USENIX18_217595}), or a hybrid pass (e.g.,~\citet{KouvarosBHL25}).
 
\paragraph{Bound propagation for a box domain}
Consider a neuron $n$ whose computation is lower bounded by $y\mapsto\sum_{i=1}^kw_iy_i$ for $y\in \mathcal{I}= \mathcal{D}=\prod_{i=1}^k[a_i,b_i]$. 
In the box domain, it is easy to define an input $y$ obtaining the minimum: every entry $y_i$ is on the boundaries, $a_i$ or $b_i$, depending on the sign of $w_i$. Formally, $y_i=a_i$ if $w_i\geq 0$, or $y_i=b_i$ otherwise. 

\subsection{Equivalent Bound Propagation for a Box Domain}
We begin with an equivalent formulation of bound propagation over the box domain, to allow a unified formulation for all bound propagations and simplify their extension to multi-channel inputs (shown later). Our formulations are defined over the \emph{minimum and maximum input entry contributions}.

Given $\bar{x}\in \mathcal{D}=\prod_{i=1}^k[a_i,b_i]$ and $y\mapsto\sum_{i=1}^kw_iy_i$, we write:
$\sum_{i=1}^kw_iy_i=\sum_{i=1}^kw_i\bar{x}_i+\sum_{i=1}^kw_i(y_i-\bar{x}_i)$.
Note that the minimum and maximum of this linear function over $y\in \mathcal{I}$ depend only on the second sum (the first sum is fixed). 
We define the \emph{minimum input entry contribution} of $i\in [k]$: $$d^-_i=\min (w_i(b_i-\bar{x}_i),w_i(a_i-\bar{x}_i)),$$ 
which lower bounds the term $w_i\cdot (y_i-\bar{x}_i)$, for $y_i \in [a_i,b_i]$.  
We define the \emph{maximum input entry contribution} of $i\in [k]$: $$d^+_i=\max (w_i(b_i-\bar{x}_i),w_i(a_i-\bar{x}_i)),$$
which upper bounds $w_i\cdot (y_i-\bar{x}_i)$, for $y_i \in [a_i,b_i]$.
Note that $d^-_i\leq 0$ and $d^+_i\geq 0$.

Given this notation, if
 $n$'s function is lower bounded by $y\mapsto\sum_{i=1}^kw_iy_i$ for $y\in   \mathcal{I}=\mathcal{D}$, 
the lower bound depends on the sum of all $d_i^-$:  
$$l=\min_{y\in \mathcal{D}}\sum_{i=1}^kw_iy_i=\sum_{i=1}^kw_i\bar{x}_i+\sum_{i=1}^kd^-_{i}$$
If $n$'s function is upper bounded by $y\mapsto\sum_{i=1}^kw_iy_i$ for $y\in \mathcal{D}$,
the upper bound depends on the sum of all $d_i^+$:
$$u=\max_{y\in \mathcal{D}}\sum_{i=1}^kw_iy_i=\sum_{i=1}^kw_i\bar{x}_i+\sum_{i=1}^kd^+_{i}$$

\paragraph{Example} 
\Cref{fig:bound} shows a network with inputs $y_1,y_2,y_3$ and an output $o_1$. The weights are on the edges and the bias of $o_1$ is above it. The neurons $n_1,n_2,o_1$ compute the sum of their bias and the products of their inputs and respective weights. The neurons $\text{ReLU}_1,\text{ReLU}_2$ compute the activation function \texttt{ReLU}$_i(n_i)=\max(0,n_i)$. Our goal is to  prove $o_1\geq 0$ for every input in the $\ell_0$-ball of $\bar{x}=(-0.3,0,0.65)$, for $\mathcal{D}= [-1,1]^3$ and $t=2$ (\Cref{fig:pertubation_space2}).  
The box bound propagation overapproximates the $\ell_0$-ball by its bounding box $\mathcal{D}$ (although $\mathcal{B}^0_t(\bar{x})$ is significantly smaller). That is, $y_i\in [-1,1]$ for $i\in[3]$. 
Then, a linear relaxation lower and upper bounds every neuron by linear functions.
In this example, we use the linear relaxations of ReLU proposed by~\citet{ZhangWCHD18,SinghGPV19}.
Then, the box bound propagation computes the minimum and maximum of these linear functions, as described.
This results in the intervals marked by $\mathcal{D}$ in \Cref{fig:bound}. 
For example, the linear functions bounding $n_1$ are exactly its function (since it is linear):
$2\cdot y_1-3\cdot y_2+7\cdot y_3$. Since  $\bar{x}=(-0.3,0,0.65)\in [-1,1]^3$, we get
  $d_1^-=\min (2\cdot (1+0.3),2\cdot(-1+0.3))=-1.4$ and similarly
 $d_2^-=-3$, $d^-_3=-11.55$. Thus, its lower bound is: 
 \begin{align*}
 \min_{(y_1,y_2,y_3)\in [-1,1]^3}2\cdot y_1-3\cdot y_2+7\cdot y_3=\hspace{2cm}\\
 \overbrace{(2\cdot (-0.3)-3\cdot 0+7\cdot 0.65)}^{3.95}-1.4-3-11.55=-12
 \end{align*}
Since the box bound propagation computes $o_1\in [-1,32]$, it cannot show that the property $o_1\geq 0$ holds.

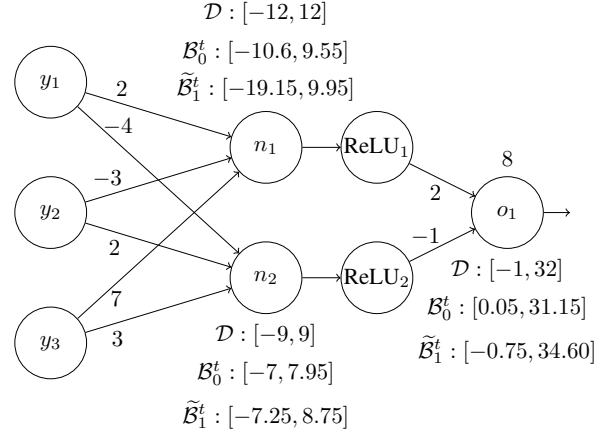
\begin{figure}[t]
\centering
\resizebox{0.95\columnwidth}{!}
{
\begin{tikzpicture}[
    node/.style={circle, draw, minimum size=1.1cm, inner sep=0pt},
]

\node[node] (X1) at (0,2) {$y_1$};
\node[node] (X2) at (0,0) {$y_2$};
\node[node] (X3) at (0,-2) {$y_3$};

\node[node] (N1) at (3.3,1) {$n_1$};
\node[node] (N2) at (3.3,-1) {$n_2$};

\node[node] (R1) at (5,1) {$\text{ReLU}_1$};
\node[node] (R2) at (5,-1) {$\text{ReLU}_2$};

\node[node] (O1) at (7,0) {$o_1$};

\draw[->] (X1) -- node[pos=0.25, above] {$2$} (N1);
\draw[->] (X1) -- node[pos=0.25, above] {$-4$} (N2);

\draw[->] (X2) -- node[pos=0.15, above] {$-3$} (N1);
\draw[->] (X2) -- node[pos=0.20, below] {$2$} (N2);

\draw[->] (X3) -- node[pos=0.24, below] {$7$} (N1);
\draw[->] (X3) -- node[pos=0.22, below] {$3$} (N2);

\draw[->] (N1) -- (R1);
\draw[->] (N2) -- (R2);

\draw[->] (R1) -- node[pos=0.4, below] {$2$} (O1);
\draw[->] (R2) -- node[pos=0.25, above] {$-1$} (O1);

\draw[->] (O1.east) -- ++(0.4,0);

\node[above=0.03 of N1] {$\widetilde{\mathcal{B}}_1^t:[-19.15,9.95]$};
\node[above=0.6 of N1] {$\mathcal{B}_0^t:[-10.6,9.55]$};
\node[above=1.2 of N1] {$\mathcal{D}:[-12,12]$};

\node[below=0.03 of N2] {$\mathcal{D}:[-9,9]$};
\node[below=0.6 of N2] {$\mathcal{B}_0^t:[-7,7.95]$};
\node[below=1.2 of N2] {$\widetilde{\mathcal{B}}_1^t:[-7.25,8.75]$};

\node[above=0.03 of O1] {$8$};
\node[below=0.03 of O1] {$\mathcal{D}:[-1,32]$};
\node[below=0.6 of O1] {$\mathcal{B}_0^t:[0.05,31.15]$};
\node[below=1.2 of O1] {$\widetilde{\mathcal{B}}_1^t:[-0.75,34.60]$};

\end{tikzpicture}
}
\caption{The three approaches for bound propagation over $\mathcal{B}_0^{t}(\bar{x})$, where $\bar{x}=(-0.3,0,0.65)$,  $\mathcal{D}=[-1,1]^3$, and $t=2$.}\label{fig:bound}
\label{fig:bound_prop_example}
\end{figure}

\subsection{Bound Propagation for an $\ell_0$-Ball}
We next show a bound propagation that precisely captures an $\ell_0$-ball input subspace. 
That is, it computes the minimum and maximum of a linear function over an $\ell_0$-ball, which coincide with those over its convex hull. 
We call it \emph{top-$t$}, since it relies on the \emph{top-$t$} input entry contributions. This idea has been proposed by~\citet{Chiang2020Certified,NEURIPS2020_0cbc5671}, however they assume $\mathcal{D}=[0,1]^v$ and single-channel inputs. 

As before, the goal is to compute the minimum of a neuron $n$ whose function is lower bounded by $y\mapsto\sum_{i=1}^kw_iy_i$.
Again, we write 
 $\sum_{i=1}^kw_i\bar{x}_i+\sum_{i=1}^kw_i(y_i-\bar{x}_i)$,  however at least $k-t$ of the terms in the second sum are $0$ since $y\in \mathcal{B}_0^t(\bar{x})$.
 Recall that $w_i(y_i-\bar{x}_i)\geq d_i^-$, for $y_i \in [a_i,b_i]$.
 Since at most $t$ entries are not $0$,
 the second sum is lower bounded by the sum of the $t$ lowest $d_i^-$. 
Thus, we sort them in a non-decreasing order, $d^-_{l_1}\leq d^-_{l_2}\leq\dots\leq d^-_{l_k}$. 
The minimum is obtained by $y$ where $y_{l_i}$ is:
 (1)~$a_{l_i}$ if $i\in[t]$ and $w_{l_i}\geq 0$, (2)~$b_{l_i}$ if $i\in[t]$ and $w_{l_i}< 0$,  and (3)~$\bar{x}_{l_i}$ if $i\notin[t]$.    
The lower bound is: 
    $$l=\min_{y\in \mathcal{B}_0^t(\bar{x})}\sum_{i=1}^kw_iy_i=\left(\sum_{i=1}^kw_i\bar{x}_i\right)+(d^-_{l_1}+\dots+ d^-_{l_t})$$

Similarly, the upper bound depends on the $t$ highest $d_i^+$:
 $$u=\max_{y\in \mathcal{B}_0^t(\bar{x})}\sum_{i=1}^kw_iy_i=\left(\sum_{i=1}^kw_i\bar{x}_i\right)+(d^+_{u_1}+\dots+ d^+_{u_t})$$ for 
 the non-increasing series $d^+_{u_1}\geq d^+_{u_2}\geq \dots\geq d^+_{u_k}$  of $d^+_i$.

\paragraph{Example}
The intervals marked by $\mathcal{B}_0^t$ in~\Cref{fig:bound} show the top-$t$ bound propagation.
For $n_1$, we showed $d_1^-=-1.4$, $d_2^-=-3$, $d_3^-=-11.55$. Thus, its $t=2$ lowest values are 
$d_2^-$ and $d_3^-$ and its lower bound is: $\min_{y\in \mathcal{B}_0^t(\bar{x})}\sum_{i=1}^kw_iy_i=3.95-3-11.55=-10.6$.
In this example, $o_1\in [0.05,31.15]$, thus top-$t$ successfully proves $o_1\geq 0$.

\begin{observation}
  The top-$t$ bound propagation computes the minimum and maximum over $\mathcal{B}_0^t(\bar{x})$ and over its convex hull.
\end{observation}
This follows since the minimum and maximum of a linear function over a compact set (in our setting, $\mathcal{B}_0^t(\bar{x})$) are equal to those over its convex hull. Namely, linear bound propagation over a non-convex perturbation space or over its convex hull are the same.
Thus, the top-$t$ bound propagation over $\mathcal{B}_0^t(\bar{x})$ coincides with the same propagation over $\mathcal{D}\cap \widetilde{\mathcal{B}}_1^t(\bar{x})$.

\subsection{Bound Propagation for $\widetilde{\mathcal{B}}_1^t(\bar{x})$}
We next define the \emph{$t$-times-top} bound propagation over a $\widetilde{\mathcal{B}}_1^t(\bar{x})$ input subspace.
The goal is to contrast it with the top-$t$ bound propagation, since 
\Cref{thm2} shows that the volumes of $\widetilde{\mathcal{B}}_1^t(\bar{x})$ and $\mathcal{D}\cap \widetilde{\mathcal{B}}_1^t(\bar{x})$ are almost the same, for large $k$.

As before, we write 
 $\sum_{i=1}^kw_i\bar{x}_i+\sum_{i=1}^kw_i(y_i-\bar{x}_i)$, for $y\in \widetilde{\mathcal{B}}_1^t(\bar{x})$, and define its minimum. By definition of $\delta_{\bar{x}}^i(y)$: 
 \begin{multline*}
   w_i(y_i-\bar{x}_i)=\\
   \underbrace{w_i(\mathbf{1}_{\{y_i>\bar{x}_i\}}(b_i-\bar{x}_i)+\mathbf{1}_{\{y_i<\bar{x}_i\}}(a_i-\bar{x}_i))}_{\geq d_i^-}\delta_{\bar{x}}^i(y)  
 \end{multline*}
 Thus, 
 $\sum_{i=1}^kw_i(y_i-\bar{x}_i)\geq {\sum_{i=1}^k d_i^- \delta_{\bar{x}}^i(y) }$. Since $\delta^i_{\bar{x}}(y)\geq 0$, $ {\sum_{i=1}^k d_i^- \delta_{\bar{x}}^i(y) }\geq d_{l_1}^- \sum_{i=1}^k  \delta_{\bar{x}}^i(y) $, for $d_{l_1}^-=\min_{i\in[k]}d_i^-$.
Since $d^-_{l_1}\leq 0$ and $y\in \widetilde{\mathcal{B}}_1^t(\bar{x})$: 
 $  d_{l_1}^- \sum_{i=1}^k  \delta_{\bar{x}}^i(y) \geq d_{l_1}^- \cdot t$.
This minimum is obtained by $y$ where $y_{l_i}=\bar{x}_{l_i}$, for
$i\neq 1$, and 
$y_{l_1}=\bar{x}_{l_1}+t\cdot (c - \bar{x}_{l_1})$, for $c=a_{l_1}$ if $w_{l_1}\geq 0$, and 
$c=b_{l_1}$ otherwise.
Thus, the lower bound depends on the product of $t$ and the lowest $d_i^-$ and is equal to:  
 $$l=\min_{y\in \widetilde{\mathcal{B}}_1^t(\bar{x})}\sum_{i=1}^kw_iy_i=\left(\sum_{i=1}^kw_i\bar{x}_i\right)+t\cdot d^-_{l_1}$$

Similarly, the upper bound depends on the product of $t$ and the highest $d_i^+$: $$u=\max_{y\in \widetilde{\mathcal{B}}_1^t(\bar{x})}\sum_{i=1}^kw_iy_i=\left(\sum_{i=1}^kw_i\bar{x}_i\right)+t\cdot d^+_{u_1}$$

 \paragraph{Example}
The intervals marked by $\widetilde{\mathcal{B}}_1^t$ in~\Cref{fig:bound} show the $t$-times-top bound propagation.
 For $n_1$, the minimum of $d_1^-=-1.4$, $d_2^-=-3$, $d_3^-=-11.55$ is $d_3^-$.
 Thus, its lower bound is $\min_{y\in \widetilde{\mathcal{B}}_1^t(\bar{x})}\sum_{i=1}^kw_iy_i=3.95-2\cdot 11.55=-19.15$.
 This bound propagation cannot prove $o_1\geq 0$.
 
 \subsection{Multi-Channel Inputs}
 Our formulation of the bound propagations as functions of the input entry contributions $d_i^-$ and $d_i^+$ naturally extends to multi-channel inputs.
In this setting, given $i\in[k]$, $d_i^-$ and $d_i^+$ sum over all channels' contributions.
Given $\bar{x}\in \mathcal{D}$ and a linear function  $y\mapsto\sum_{i=1}^k\sum_{j=1}^d w_i^{(j)}{y}_i^{(j)}$,
we define $d^-_i=\sum_{j=1}^d\min (w_{i}^{(j)}(b_{i}^{(j)}-\bar{x}_{i}^{(j)}),w_{i}^{(j)}(a_{i}^{(j)}-\bar{x}_{i}^{(j)}))$
and
$d^+_i=\sum_{j=1}^d\max (w_{i}^{(j)}(b_{i}^{(j)}-\bar{x}_{i}^{(j)}),w_{i}^{(j)}(a_{i}^{(j)}-\bar{x}_{i}^{(j)}))$.
The lower and upper bounds of all bound propagations are the same as for single-channel inputs, except that the first sum is $\sum_{i=1}^k\sum_{j=1}^d w_i^{(j)}{\bar{x}}_i^{(j)}$. 
 
\subsection{Comparison }
 Our formulation explains why a bound propagation over the $\ell_0$-ball is tighter than over our $\ell_1$-like polytope or $\mathcal{D}$. All bound propagations have a shared term (i.e., $\sum_{i=1}^kw_i\bar{x}_i$, for single-channel inputs) and they differ in the term they add: for the lower bound,
top-$t$ adds the sum of the $t$ lowest minimum input entry contributions, while $t$-times-top adds the product of $t$ and the lowest minimum contribution, and the box bound propagation adds the sum of all minimum contributions (and similarly for the upper bound).
The $t$-times-top and box bound propagations are incomparable for $t>1$, because $t$ times the top-1 value can be smaller or greater than the sum of all values. For example, in~\Cref{fig:bound_prop_example}, the box bound propagation has a tighter lower bound of $n_1$ than $t$-times-top, but a looser upper bound.  
The theoretical time complexity of our bound propagations is linear in $k$, for single-channel inputs, or $k \cdot d $, for multi-channel inputs, like the common box bound propagation.
In practice, the bound propagations run on a GPU, which affects the complexity (described in the appendix). 

\section{Evaluation}\label{sec:eval}
In this section, we evaluate the top-$t$ bound propagation. We empirically show that top-$t$ is more precise than the other bound propagations for $\ell_0$-balls and that it significantly boosts \tool, the state-of-the-art $\ell_0$ robustness verifier.

\newcommand{\topt}{\texttt{top-$t$-GP}\xspace}
\newcommand{\topo}{\texttt{$t$-times-top-GP}\xspace}

\paragraph{Setup}
We extended GPUPoly~\cite{MullerS0PV21}, which \tool calls a large number of times (sometimes $10^6$).
GPUPoly verifies using box bound propagation with the linear relaxation of~\citet{SinghGPV19}.
We extended it with our bound propagations: top-$t$ (called \topt) and $t$-times-top (\topo).
The implementation is described in the appendix, which also shows that 
the running time of \topt is similar to GPUPoly. 
We ran experiments on a dual AMD EPYC 7713 server with 2TB RAM and eight A100 GPUs. 
We consider the same image classifiers as \tool, for MNIST~\cite{LeCunBBH98}, Fashion-MNIST~\cite{abs-1708-07747}, and CIFAR-10~\cite{krizhevsky09}.
MNIST and Fashion-MNIST consist of single-channel images (grayscale) and CIFAR-10 consists of three-channel images (RGB).
The network architectures are provided in ERAN~\cite{ethsri-eran}.

\paragraph{Bound propagation only}
GPUPoly cannot directly prove robustness for $\ell_0$-balls, since it runs box bound propagation and thus requires to overapproximate an $\ell_0$-ball with its bounding box. However, the bounding box is $\mathcal{D}_v$ and networks are not robust in $\mathcal{D}_v$ since not all inputs are classified the same. 
Although \topt improves GPUPoly's precision for $\ell_0$-balls, we next show that it is still typically not precise enough to prove robustness in an $\ell_0$-ball where all pixels can be perturbed (i.e., $\mathcal{K}=[v]$).
We show that even for $t=1$ (the smallest $\ell_0$-ball), it often loses too much precision to prove robustness. 
In this experiment, for each network, we consider the first 100 images from its test set. 
For each image, we check whether the network classifies it correctly. If so,
we run \topt for its $\ell_0$-ball with $t=1$. 
\Cref{table:fast_cert} shows the number of images that are correctly classified (\textbf{\# CC}), the number of robust $\ell_0$-balls (\textbf{\# R}), and the number of $\ell_0$-balls that \topt verifies. 
We obtain the number of robust $\ell_0$-balls from Calzone~\cite{Shapira23}, a complete (exact) verifier for $\ell_0$-balls.
The table shows that except for MNIST classifiers,
\topt cannot verify even a single $\ell_0$-ball.
These results imply that \topt cannot, on its own, prove robustness in $\ell_0$-balls with $\mathcal{K} = [v]$.
However, we later show that it boosts \tool, which does prove robustness in $\ell_0$-balls with $\mathcal{K} = [v]$. \tool decomposes the verification task into multiple robustness verification tasks over
$\ell_0$-balls defined over subsets of pixels ($\mathcal{K}\subseteq [v]$) whose size is at most $200$ (a more detailed description is provided later).
The next experiment explains why \topt boosts \tool. 
 
\begin{table}[t]
\centering
{\fontsize{9}{11}\selectfont
\begin{tabular}{llcccc}
\toprule
\cmidrule(lr){3-5}
\textbf{Dataset} & \textbf{Network} &  \textbf{\# CC} & \textbf{\# R} & \topt \\
\midrule
MNIST & $6\times200$\_PGD &  99 & 97 & 94 \\
\cmidrule(lr){2-5}
      & ConvSmall &  98 & 95 & 36 \\
\cmidrule(lr){2-5}
      & ConvSmallPGD &  94 & 93 & 57 \\
\cmidrule(l){2-5}
      & ConvMedPGD &  99 & 94 & 43 \\
\cmidrule(lr){2-5}
      & ConvBig &  97 & 94 & 0 \\
\midrule
Fashion & ConvSmallPGD &  83 & 75 & 0 \\
\cmidrule(lr){2-5}
        & ConvMedPGD &  86 & 83 & 0 \\
\midrule
CIFAR-10 & ConvSmallPGD &  70 & 61 & 0 \\
\bottomrule
\end{tabular}
}
\caption{The number of verified $\ell_0$-balls for $t=1$. }
\label{table:fast_cert}
\end{table}

\begin{figure*}[t]
\centering
\includegraphics[width=0.95\textwidth]{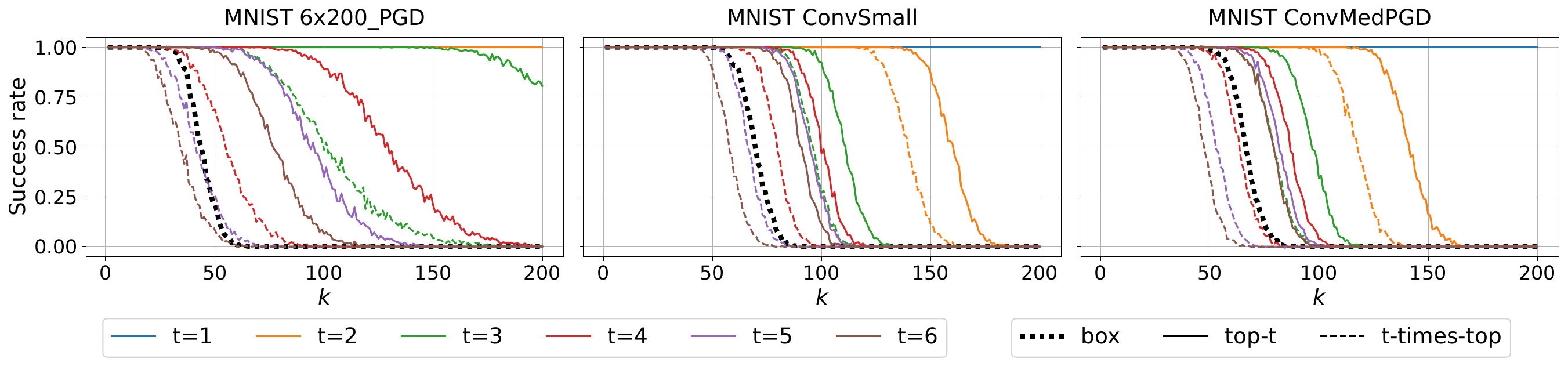}
\caption{The success rate of the three bound propagations over $\mathcal{B}_0^t(\bar{x})$ for different choices of $\mathcal{K}\subseteq [v]$, as a function of $k=|\mathcal{K}|$.}
\label{fig:bound_prop}
\end{figure*}

\paragraph{Bound propagation comparison}
We next show that \topt is more precise than the other bound propagations for verifying robustness in $\ell_0$-balls over subsets of pixels ($|\mathcal{K}|<v$).  
We consider three networks, three respective images, $k\in [200]$, and $t\in [6]$. For each network, image, and $k$, we randomly sample $500$ subsets $S\subseteq [v]$ of size $k$ and let each bound propagation verify robustness in the neighborhood
$\mathcal{B}_0^t(\bar{x})$ where $\mathcal{K}=S$, for each $t$.
\Cref{fig:bound_prop} shows the success rate (the ratio of $\mathcal{B}_0^t(\bar{x})$ proven robust out of all 500 analyzed) as a function of $k$, for each approach and each~$t$.
The results show that \topt's success rate is never lower than GPUPoly (with the box bound propagation) and often improves it, especially for small $t$ or large $k$. It is also significantly better than \topo for $t>1$, despite the negligible volume difference shown in \Cref{thm2}. 
In fact, for $t>4$, even the box bound propagation is better than \topo, although the volume of $\widetilde{\mathcal{B}}_1^t(\bar{x})$ is significantly smaller than $\mathcal{D}$ (for large $k$).
This shows that bound propagation highly depends on the shape of the property’s input subspace and not only on its volume.

\begin{figure*}[t]
\centering
\includegraphics[width=0.95\textwidth]{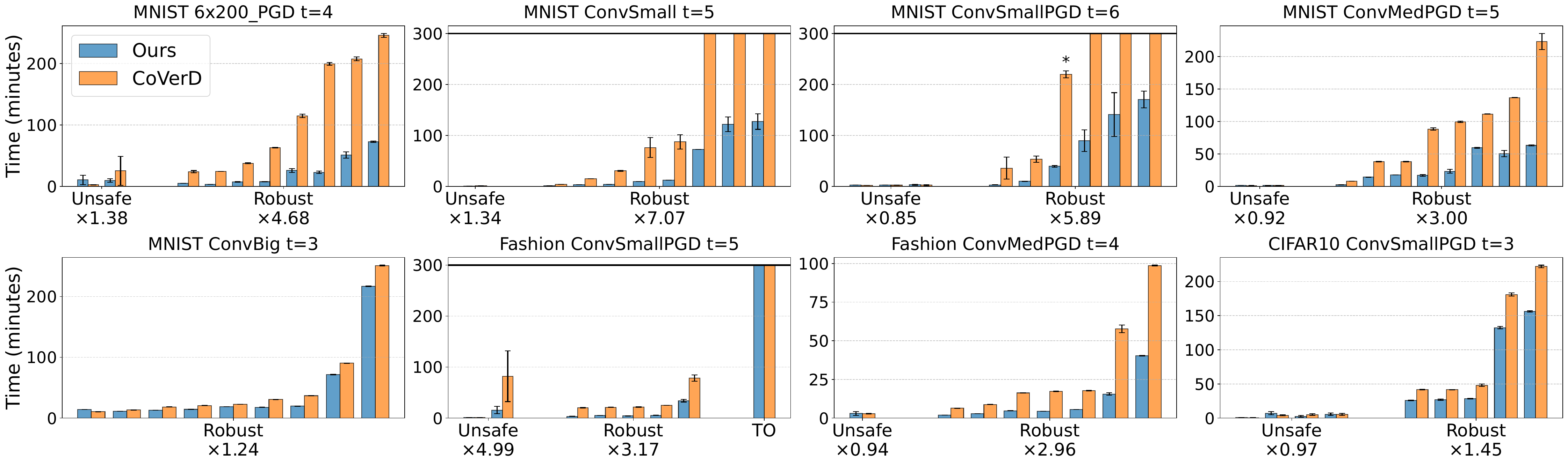}
\caption{{\tool} + \topt compared to \tool on its most challenging benchmarks.}\label{fig:covev}
\end{figure*}

\paragraph{Boosting $\ell_0$ certification}
Lastly, we show that \topt boosts the performance of \tool~\cite{ShapiraWSD24}, which verifies robustness in $\ell_0$-balls where all pixels can be perturbed ($\mathcal{K}=[v]$). 
\tool is a sound and complete local robustness verifier for $\ell_0$-balls. It is the state-of-the-art, improving Calzone~\cite{Shapira23}, which in turn outperforms a MILP-based verification.  
A naive verifier would submit $\binom{v}{t}$ box neighborhoods to a verifier (e.g., GPUPoly) and determine robustness if all are robust.
However, verifying $\binom{v}{t}$ neighborhoods is infeasible for $t>2$.
Instead, \tool defines sets $S_1,\ldots,S_R$ of size $k$ that cover all subsets of $t$ pixels over $[v]$.
For each $S_r$, it submits the box neighborhood $\{y\in\mathcal{D}_v\mid \forall i\notin S_r :y_i= \bar{x}_i\}$ to GPUPoly.
If it is robust, \tool continues to $S_{r+1}$; otherwise, it refines $S_r$ to smaller subsets and submits their neighborhoods to GPUPoly. 
Generally, the larger the sets $S_r$ whose neighborhoods are proven robust, the fewer calls to GPUPoly.
With \topt, \tool can submit $\mathcal{B}_0^t(\bar{x})$ where $\mathcal{K}=S_r$.
Also, since \topt has higher success rate than GPUPoly for large $k$ (\Cref{fig:bound_prop}), \tool can verify significantly larger sets $S_r$, which boosts its performance. 

We evaluate on the most challenging benchmarks of \tool: for each network, we focus on the largest evaluated $t$.
The appendix includes experiments over less challenging benchmarks of \tool.
We ran the latest version of \tool, with the original hyper-parameters, except that we ran with five MILP instances.  
We used eight GPUs, as in \tool's experiments (the running time decreases roughly linearly with the number of GPUs).  
The timeout of each verifier is five hours for each $\ell_0$-ball.
Since \tool relies on random choices, we repeat each experiment three times with different seeds to show that our speedup is independent of the choice of seed.
\Cref{fig:covev} reports the average execution time and standard deviation. Each plot corresponds to a network.
A pair of bars corresponds to an $\ell_0$-ball where each bar shows the execution time in minutes (the y-axis), for \tool and ours (\tool+ \topt).
The horizontal line marks the timeout. Bars reaching this line indicate timeouts, in all three runs.
For a single $\ell_0$-ball, one run of \tool out of three reached the timeout (denoted by *) and is excluded from its average and standard deviation.
The x-axis begins with the \emph{unsafe} $\ell_0$-balls (non-robust), if they exist, followed by the \emph{robust} $\ell_0$-balls. One plot also has a timeout instance (\emph{TO}). 
Within each group,
 the $\ell_0$-balls are sorted by the execution time of {\tool} to better visualize the time differences.  
 Below each group, we show the ratio between the sums of the average execution times of {\tool} and ours, excluding $\ell_0$-balls where \tool timed out in all three runs.
 The results show that \topt boosts \tool's performance on its robustness benchmarks by 1.24x-7.07x, with a geometric mean of 3.16. 
It is sometimes slower than \tool, but only when both verifiers terminate within minutes (as opposed to hours in the more challenging $\ell_0$-balls).
In the appendix, we demonstrate that \tool+ \topo is substantially slower than our approach, consistent with the results  in~\Cref{fig:bound_prop}.

\section{Related Work}
Many verifiers analyze the local robustness of neural networks over convex perturbation spaces, e.g., $\ell_p$-balls for $p\geq 1$~\cite{TjengXT19,ZhangWCHD18,MullerS0PV21,GehrMDTCV18,KatzBDJK17,Leino21,Huang21,NEURIPS2021_26657d5f,WuZ21}, where many rely on bound propagation and convex relaxations to scale.
\citet{Chiang2020Certified, NEURIPS2020_0cbc5671} propose bound propagation for $\ell_0$-balls, which is similar to ours but is limited to single-channel inputs over $[0,1]^v$.
Our experiments show that this bound propagation often fails to prove robustness  even when only a single pixel can be perturbed (i.e., $t=1$).
\citet{NEURIPS2021_26657d5f} propose a tight bound propagation for simplex input spaces.  
Though it can precisely capture our asymmetrically scaled $\ell_1$-ball, it currently does not support the intersection with a box, required to precisely capture the convex hull.
\citet{jia-etal-2019-certified} verify robustness to word substitutions using interval bound propagation (IBP), where $t=v$.
\citet{huang-etal-2019-achieving} propose IBP for $t<v$. 
They overapproximate the convex hull of the $\ell_0$-ball with a simplex 
and say that manipulating their convex hull is infeasible (whereas we explicitly characterize ours). 
For a similar setting, \citet{NEURIPS2020_0cbc5671} verify using dynamic programming. 
\citet{Chiang2020Certified, Levine2020Patch} propose certified defenses against patch attacks, where the attacker perturbs a subset of pixels within a contiguous region. In contrast, we focus on few-pixel attacks, where the attacker can perturb any subset of pixels, and the allowed number of perturbed pixels is typically smaller than that in patch attacks.
\section{Conclusion}
We show that the convex hull of an $\ell_0$-ball is the intersection of its bounding box and an $\ell_1$-like polytope.
We present a bound propagation over an $\ell_0$-ball, tighter than bound propagations over its bounding box or the $\ell_1$-like polytope.
It boosts the analysis of a neural network robustness verifier for $\ell_0$-balls by 1.24x-7.07x, with a geometric mean of 3.16, on its most challenging robustness benchmarks.

\bibliography{aaai2026}

\clearpage
\appendix
\section{Proofs}\label{proofs}
\begin{lemma}\label{lem:1}
    $\widetilde{\mathcal{B}}_1^t(\bar{x})$ is convex and compact.
\end{lemma}
\begin{proof}
    
      {\bf Convexity:} 
      Let $y,y'\in \widetilde{\mathcal{B}}_1^t(\bar{x})$ and $\lambda\in(0,1)$. We prove that for $z=\lambda y +(1-\lambda)y'$, we have 
       $z\in \widetilde{\mathcal{B}}_1^t(\bar{x})$. 
       We show it by proving: $$\sum_{i=1}^k\delta_{\bar{x}}^i(z)\leq \lambda \sum_{i=1}^k\delta_{\bar{x}}^i(y) + (1-\lambda)\sum_{i=1}^k\delta_{\bar{x}}^i(y')\leq t$$
       To this end, 
      we prove that $\delta_{\bar{x}}^i(z)\leq\lambda \delta_{\bar{x}}^i(y)+(1-\lambda)\delta_{\bar{x}}^i(y')$ for all $i\in [k]$. 
     We split to cases.
     If $z_i={\bar{x}}_i$, then $\delta_{\bar{x}}^i(z)=0$ and the claim holds. Otherwise, $z_i\neq {\bar{x}}_i$.  If $y_i=y'_i$, then $z_i=y_i=y'_i$ and so $\delta_{\bar{x}}^i(z)=\lambda \delta_{\bar{x}}^i(y)+(1-\lambda)\delta_{\bar{x}}^i(y')$, and the claim holds. Otherwise, $y_i\neq y'_i$. Assume without loss of generality that $z_i>{\bar{x}}_i$ and $y_i>y'_i$ (and thus  $y_i>z_i>y'_i$, since $\lambda\in (0,1)$). 
    If $y'_i>{\bar{x}}_i$, we have: $$z_i-{\bar{x}}_i=\lambda y_i+(1-\lambda)y'_i-{\bar{x}}_i=\lambda(y_i-{\bar{x}}_i)+(1-\lambda)(y'_i-{\bar{x}}_i)$$ 
    By dividing by $b_i-{\bar{x}}_i$, we get: $$\delta_{\bar{x}}^i(z)\leq\lambda \delta_{\bar{x}}^i(y)+(1-\lambda)\delta_{\bar{x}}^i(y')$$ because $y_i,z_i,y'_i>{\bar{x}}_i$ and their denominators in $\delta_{\bar{x}}^i(y)$, $\delta_{\bar{x}}^i(z)$, $\delta_{\bar{x}}^i(y')$ are all $b_i-{\bar{x}}_i$. 
    This proves the claim.
         Otherwise, in case $y'_i\leq {\bar{x}}_i$, we have:
         \begin{multline*}
         z_i-{\bar{x}}_i=\lambda y_i+(1-\lambda)y'_i-{\bar{x}}_i\\
         =\lambda(y_i-{\bar{x}}_i)+(1-\lambda)(y'_i-{\bar{x}}_i)\leq \lambda(y_i-{\bar{x}}_i)
         \end{multline*}
          By dividing by $b_i-{\bar{x}}_i$, we get: $\delta_{\bar{x}}^i(z)\leq\lambda \delta_{\bar{x}}^i(y)$, because
           $y_i,z_i>{\bar{x}}_i$ and their denominators in $\delta_{\bar{x}}^i(y)$ and $\delta_{\bar{x}}^i(z)$ are both $b_i-{\bar{x}}_i$.
           This proves the claim.
      
{\bf Compactness:}
     We decompose $\widetilde{\mathcal{B}}_1^t({\bar{x}})$ into $2^k$ orthants:
    \begin{multline*}
        \bigcup_{s_1,\dots,s_k\in\{\pm1\}}{\{y\in \widetilde{\mathcal{B}}_1^t({\bar{x}})\mid \forall i\in[k]:s_i(y_i-{\bar{x}}_i)\geq 0\}}
    \end{multline*}
    Since a finite union of compact sets is compact, we prove that every orthant is compact.
     We denote the orthant of fixed $s_1,\dots,s_k$ by: 
    $$\widetilde{\mathcal{B}}_1^t({\bar{x}})_{s_1,\dots,s_k}\triangleq {\{y\in \widetilde{\mathcal{B}}_1^t({\bar{x}})\mid \forall i\in[k]:s_i(y_i-{\bar{x}}_i)\geq 0\}}$$
    We remind that $\widetilde{\mathcal{B}}_1^t({\bar{x}})=\{y\in\mathbb{R}^k \mid \sum_{i=1}^k\delta_{\bar{x}}^i(y)\leq t\}$ and $\delta_{\bar{x}}^{i}(y)=\frac{y_i-{\bar{x}}_i}{\mathbf{1}_{\{y_i>{\bar{x}}_i\}}(b_i-{\bar{x}}_i)+\mathbf{1}_{\{y_i<{\bar{x}}_i\}}({a_i-\bar{x}}_i)}$.
    Consider the set of inputs within the first orthant of the $\ell_1$-ball with radius $t$:
    $\Delta_{k,t}=\{z\in[0,\infty)^k\mid\sum_{i=1}^kz_i\leq t\}$, which is compact.
    We show that $\Delta_{k,t}$ can be transformed to $\widetilde{\mathcal{B}}_1^t({\bar{x}})_{s_1,\dots,s_k}$
    with an affine transformation (which in particular is continuous) and thus  $\widetilde{\mathcal{B}}_1^t({\bar{x}})_{s_1,\dots,s_k}$ is compact as a continuous image of a compact set.  
    The affine transformation from $\Delta_{k,t}$ to $\widetilde{\mathcal{B}}_1^t({\bar{x}})_{s_1,\dots,s_k}$ is: $y_i=(b_i-{\bar{x}}_i)\cdot z_i + {\bar{x}}_i$, if $s_i=1$, and $y_i=(a_i-{\bar{x}}_i)\cdot z_i + {\bar{x}}_i$, if $s_i=-1$. This concludes the proof. 

\end{proof}

\begin{lemma}\label{lem:2}
$E_\cap \subseteq \mathcal{B}_0^t({\bar{x}})$, where $E_\cap$ is the extreme points of $\mathcal{D}\cap\widetilde{\mathcal{B}}_1^t({\bar{x}})$. 

\end{lemma}
\begin{proof}
We show that $E_\cap$ is contained in a subset of $\mathcal{B}_0^t({\bar{x}})$ consisting of all inputs $z$ with at least $k-t$ entries $z_i={\bar{x}}_i$ and the other entries are the bounds $a_i$ or $b_i$:
$$
   \{z\in \prod_{i=1}^k\{a_i,{\bar{x}}_i,b_i\}\mid |\{i\in[k]\mid z_i\neq {\bar{x}}_i\}| \leq t\} 
 $$
  Let $y\in E_\cap$.
  We show: (1)~$y\in \prod_{i=1}^k\{a_i,{\bar{x}}_i,b_i\}$ 
  and (2) if $y\in \prod_{i=1}^k\{a_i,{\bar{x}}_i,b_i\}$, then $|\{i\in[k]:y_i\neq {\bar{x}}_i\}| \leq t$.\\
   Proof for (2):
    Since $y\in E_\cap\subseteq \widetilde{\mathcal{B}}_1^t({\bar{x}})$, we get
     $\sum_{i=1}^k\delta_{\bar{x}}^i(y)\leq t$. Since $y\in \prod_{i=1}^k\{a_i,{\bar{x}}_i,b_i\}$, for every $i$, either $y_i={\bar{x}}_i$ and $\delta_{\bar{x}}^i(y)=0$ or 
     $y_i\in \{a_i,b_i\}$ and then $\delta_{\bar{x}}^i(y)=1$. Thus, $|\{i\in[k]:y_i\neq {\bar{x}}_i\}| \leq t$.\\
   Proof for (1): Assume to the contrary that there exists $l$ such that $y_{l}\notin \{a_l,{\bar{x}}_l,b_l\}$. Since $y\in E_\cap\subseteq \mathcal{D}$, we have $\delta_{\bar{x}}^i(y)\in[0,1]$  for every $i\in[k]$. As $y_{l}\notin \{a_l,{\bar{x}}_l,b_l\}$ we also have $\delta_{\bar{x}}^l(y)\in(0,1)$. 
     We show that in this case $y$ is a convex combination of two different inputs $y',y'' \in \mathcal{D}\cap\widetilde{\mathcal{B}}_1^t({\bar{x}})$
     and thus $y$ cannot be an extreme point, in contradiction. 
     We define these two points symmetrically, i.e., $y=0.5(y'+y'')$. 
     We remind that since $y\in E_\cap\subseteq \widetilde{\mathcal{B}}_1^t({\bar{x}})$, we have
     $\sum_{i=1}^k\delta_{\bar{x}}^i(y)\leq t$. Accordingly, we split to cases:
    \begin{itemize}[nosep,nolistsep]
        \item $\sum_{i=1}^k\delta_{\bar{x}}^i(y)<t$: 
        In this case, $y$ is within $\widetilde{\mathcal{B}}_1^t({\bar{x}})$ (not on its boundary). We identify a small value that can be added and subtracted from $y_{l}$ such that the new inputs do not cross ${\bar{x}}_l$ (possible since $y_{l}\neq {\bar{x}}_{l}$ and thus $\delta_{\bar{x}}^l(y)>0$), do not cross the bound that is closer to $y_{l}$ than to ${\bar{x}}_{l}$ (possible since $y_{l}\neq a_l,b_l$ and thus $\delta_{\bar{x}}^l(y)<1$) and do not cross the boundary of $\widetilde{\mathcal{B}}_1^t({\bar{x}})$ (possible since $\sum_{i=1}^k\delta_{\bar{x}}^i(y)<t$).
        Formally, $y'=y+\epsilon\cdot v_{l}$ and $y''=y-\epsilon \cdot v_{l}$ for:
        $$v_l=(\mathbf{1}_{\{y_{l}>{\bar{x}}_{l}\}}(b_{l}-{\bar{x}}_{l})+\mathbf{1}_{\{y_{l}<{\bar{x}}_{l}\}}(a_{l}-{\bar{x}}_{l}))e_l$$
        where $e_{l}$ is the standard unit vector with $1$ at entry $l$ and $0$ elsewhere, and
        \begin{align*}
        \epsilon =\min\{\delta_{\bar{x}}^l(y),1-\delta_{\bar{x}}^l(y),t-\sum_{i=1}^k\delta_{\bar{x}}^i(y)\}
        \end{align*}
        As $\epsilon>0$ and $v_l$ is not the zero vector (since $y_l$ lies strictly between $\bar{x}_l$ and one of the bounds $a_l$ or $b_l$) we have $y'\neq y''$ and by the construction $y',y''\in \mathcal{D}\cap\widetilde{\mathcal{B}}_1^t({\bar{x}})$.
        
        \item $\sum_{i=1}^k\delta_{\bar{x}}^i(y)=t$: 
        In this case, $y$ is on the boundary of $\widetilde{\mathcal{B}}_1^t({\bar{x}})$. We show there is another index $m$ such that $y_{m}\notin \{a_m,{\bar{x}}_m,b_m\}$ and 
        identify small values that can be added and subtracted from $y_{l},y_m$ such that the new inputs do not cross ${\bar{x}}_l,{\bar{x}}_m$ (possible since $y_{l}\neq {\bar{x}}_{l},y_m\neq {\bar{x}}_m$ and thus $\delta_{\bar{x}}^l(y),\delta_{\bar{x}}^m(y)>0$), do not cross the bound that is closer to $y_{l}$ and $y_m$ than to ${\bar{x}}_{l}$ and ${\bar{x}}_m$ (possible since $y_{l}\neq a_l,b_l$ and $y_m \neq a_m,b_m$ and thus $\delta_{\bar{x}}^l(y),\delta_{\bar{x}}^m(y)<1$) and remain on the boundary of $\widetilde{\mathcal{B}}_1^t({\bar{x}})$.  
        We have $\delta_{\bar{x}}^{l}(y)\in(0,1)$ and since $\sum_{i=1}^k\delta_{\bar{x}}^i(y)=t$ for an integer $t$, there must be $m\neq l$ such that $\delta_{\bar{x}}^{m}(y)\in(0,1)$, and $y_{m}\notin \{a_m,{\bar{x}}_m,b_m\}$. 
         We define $y'=y+\epsilon\cdot{v_{lm}}$ and $y''=y-\epsilon\cdot {v_{lm}}$, where ${v_{lm}}$ changes coordinates $l$ and $m$ in a scaled complementary manner to preserve $\sum_{i=1}^k\delta_{\bar{x}}^i(y)=t$, formally $v_{lm}=v_{l}-v_{m}$ for:
         \begin{align*}
             {v_{l}}=(&\mathbf{1}_{\{y_{l}>{\bar{x}}_{l}\}}(b_{l}-{\bar{x}}_{l})\\
             &+\mathbf{1}_{\{y_{l}<{\bar{x}}_{l}\}}(a_{l}-{\bar{x}}_{l})) \cdot e_{l}
         \end{align*}
         and similarly: 
         \begin{align*}
             {v_{m}}=(&\mathbf{1}_{\{y_{m}>{\bar{x}}_{m}\}}(b_{m}-{\bar{x}}_{m})\\
             &+\mathbf{1}_{\{y_{m}<{\bar{x}}_{m}\}}(a_{m}-{\bar{x}}_{m})) \cdot e_{m}
         \end{align*}
        and:
        \begin{align*}
            \epsilon =\min\{\delta_{\bar{x}}^l(y),1-\delta_{\bar{x}}^l(y),\delta_{\bar{x}}^m(y), 1-\delta_{\bar{x}}^m(y)\}
        \end{align*}
        As $\epsilon>0$ and $v_{lm}$ is not the zero vector we have $y'\neq y''$ and by the construction $y',y''\in \mathcal{D}\cap\widetilde{\mathcal{B}}_1^t({\bar{x}})$.
    \end{itemize}
  
\end{proof}

\ftb*
\begin{proof}
    We decompose $\widetilde{\mathcal{B}}_1^t({\bar{x}})$ into $2^k$ orthants:
    \begin{multline*}
        \bigcup_{s_1,\dots,s_k\in\{\pm1\}}{\{y\in \widetilde{\mathcal{B}}_1^t({\bar{x}})\mid \forall i\in[k]:s_i(y_i-{\bar{x}}_i)\geq 0\}}
    \end{multline*}
    where the intersection of difference sets in this union has zero volume. 
    We next determine the volume of a single orthant, with a fixed $s_1,\dots,s_k$, denoted by 
    $\widetilde{\mathcal{B}}_1^t({\bar{x}})_{s_1,\dots,s_k}\triangleq {\{y\in \widetilde{\mathcal{B}}_1^t({\bar{x}})\mid \forall i\in[k]:s_i(y_i-{\bar{x}}_i)\geq 0\}}$.
    We remind that $\widetilde{\mathcal{B}}_1^t({\bar{x}})=\{y\in\mathbb{R}^k \mid \sum_{i=1}^k\delta_{\bar{x}}^i(y)\leq t\}$ and $\delta_{\bar{x}}^{i}(y)=\frac{y_i-{\bar{x}}_i}{\mathbf{1}_{\{y_i>{\bar{x}}_i\}}(b_i-{\bar{x}}_i)+\mathbf{1}_{\{y_i<{\bar{x}}_i\}}({a_i-\bar{x}}_i)}$.
        Consider the set of inputs within the first orthant of the $\ell_1$-ball with radius $t$:
    $\Delta_{k,t}=\{z\in[0,\infty)^k\mid\sum_{i=1}^kz_i\leq t\}$.
    It is known that $\text{vol}(\Delta_{k,t})=\frac{t^k}{k!}$. We next transform $\Delta_{k,t}$ to $\widetilde{\mathcal{B}}_1^t({\bar{x}})_{s_1,\dots,s_k}$ by a change of variables: $y_i=(b_i-{\bar{x}}_i)\cdot z_i + {\bar{x}}_i$, if $s_i=1$, and $y_i=(a_i-{\bar{x}}_i)\cdot z_i + {\bar{x}}_i$, if $s_i=-1$.
The change of variables shows that every axis $z_i$ is scaled by $|b_i-{\bar{x}}_i|$, if $s_i=1$, or by $|a_i-{\bar{x}}_i|$, if $s_i=-1$, thus:
    \begin{multline*}
        \text{vol}(\widetilde{\mathcal{B}}_1^t({\bar{x}})_{s_1,\dots,s_k}) = \\
        \text{vol}(\Delta_{k,t}) \prod_{i=1}^k\left(\mathbf{1}_{\{s_i=1\}}|b_i-{\bar{x}}_i|+\mathbf{1}_{\{s_i=-1\}}|a_i-{\bar{x}}_i|\right)
    \end{multline*}
As mentioned, the total volume is their summation. We observe that the summation of the scaling is $\text{vol}(\mathcal{D})$: 
    \begin{multline*}
    \text{vol}(\mathcal{D})= \prod_{i=1}^k({b_i-a_i})=
     \prod_{i=1}^k{\left(|b_i-{\bar{x}}_i|+|a_i-{\bar{x}}_i|\right)}=\\
        \sum_{s1,\dots,s_k\in\{\pm1\}}\prod_{i=1}^k\left(\mathbf{1}_{\{s_i=1\}}|b_i-{\bar{x}}_i|+\mathbf{1}_{\{s_i=-1\}}|a_i-{\bar{x}}_i|\right)
    \end{multline*}
    Overall, we get $\text{vol}(\widetilde{\mathcal{B}}_1^t({\bar{x}}))=\text{vol}(\mathcal{D})\frac{t^k}{k!}$.\\
    The proof for $\mathcal{D}\cap\widetilde{\mathcal{B}}_1^t({\bar{x}})$ is similar.
First, we split to the orthants: $\widetilde{\mathcal{B}}_1^t({\bar{x}})_{s_1,\dots,s_k}\cap \mathcal{D}$.
Second, we consider $\Delta_{k,t}\cap [0,1]^k$ (since $\delta_{\bar{x}}^i(y)\in [0,1]$ for $y\in \mathcal{D}$).
Third, we use the same variable change and obtain the same scaling factor:
 \begin{multline*}
        \text{vol}(\widetilde{\mathcal{B}}_1^t({\bar{x}})_{s_1,\dots,s_k}\cap\mathcal{D}) = \\
        \text{vol}(\Delta_{k,t}\cap [0,1]^k) \prod_{i=1}^k\left(\mathbf{1}_{\{s_i=1\}}|b_i-{\bar{x}}_i|+\mathbf{1}_{\{s_i=-1\}}|a_i-{\bar{x}}_i|\right)
    \end{multline*}
    Fourth, again, the summation of the scaling is $\text{vol}(\mathcal{D})$, thus we get $\text{vol}(\widetilde{\mathcal{B}}_1^t({\bar{x}})\cap\mathcal{D})=\text{vol}(\mathcal{D})\cdot \text{vol}(\Delta_{k,t}\cap[0,1]^k)$.
Fifth, the second term can be computed by the Irwin–Hall distribution \cite{Irwin, Hall}.
Technically, to compute the volume, we compute the probability that a uniformly sampled input from $[0,1]^k$ lies in $\Delta_{k,t}$, that is its sum of entries is at most $t$. 
The sum of entries of a uniformly sampled input in $[0,1]^k$ is distributed as the Irwin-Hall distribution.
Therefore, the probability we are interested in is obtained by evaluating the Irwin-Hall CDF at $t$.
Since the volume of $[0,1]^k$ is $1$, we get that $\text{vol}(\Delta_{k,t}\cap[0,1]^k)$ is exactly this probability:
$\text{vol}(\Delta_{k,t}\cap[0,1]^k)=\frac{t^k}{k!}\sum_{r=0}^{t-1}(-1)^r\binom{k}{r}(1-\frac{r}{t})^k$.
\end{proof}

\begin{lemma}\label{lem:3}
    $\widetilde{\mathcal{B}}_{1, \infty}^t(\bar{x})$ is convex and compact.
\end{lemma}
\begin{proof}
    
      {\bf Convexity:} 
      Let $y,y'\in \widetilde{\mathcal{B}}_{1, \infty}^t(\bar{x})$ and $\lambda\in(0,1)$. We prove that for $z=\lambda y +(1-\lambda)y'$, we have 
       $z\in \widetilde{\mathcal{B}}_{1, \infty}^t(\bar{x})$. 
       We show it by proving:
       \begin{multline*}
           \sum_{i=1}^k \max_{j\in[d]}\delta_{\bar{x}}^{i,j}(z)\leq \\
           \lambda \sum_{i=1}^k\max_{j\in[d]}\delta_{\bar{x}}^{i,j}(y) + (1-\lambda)\sum_{i=1}^k\max_{j\in[d]}\delta_{\bar{x}}^{i,j}(y')\leq t
       \end{multline*}  
      To this end, we prove that for all $i\in [k]$: 
      \begin{align*}
         \max_{j\in[d]}\delta_{\bar{x}}^{i,j}(z)\leq\lambda \max_{j\in[d]}\delta_{\bar{x}}^{i,j}(y)+(1-\lambda)\max_{j\in[d]}\delta_{\bar{x}}^{i,j}(y') 
      \end{align*}
     We prove it by showing that for all $i\in [k]$ and $j\in [d]$:
         $\delta_{\bar{x}}^{i,j}(z)\leq\lambda \delta_{\bar{x}}^{i,j}(y)+(1-\lambda)\delta_{\bar{x}}^{i,j}(y')$.
         The proof is identical 
         to the proof of the corresponding inequality in \Cref{lem:1} with an additional index.  
     
{\bf Compactness:}
     We decompose $\widetilde{\mathcal{B}}_{1, \infty}^t(\bar{x})$ into $(2^{d})^{k}$ orthants:
    \begin{multline*}
        \bigcup_{s_1,\dots,s_k\in\{\pm1\}^d}\\{\{y\in \widetilde{\mathcal{B}}_{1, \infty}^t(\bar{x})\mid \forall i\in[k], j\in[d]:
        s_i^{(j)}(y_i^{(j)}-{\bar{x}}_i^{(j)})\geq 0\}}
    \end{multline*}
    Since a finite union of compact sets is compact, we prove that every orthant is compact.
     We denote the orthant of fixed $s_1,\dots,s_k$ by:
     \begin{multline*}
         \widetilde{\mathcal{B}}_{1, \infty}^t(\bar{x})_{s_1,\dots,s_k}\triangleq \\ 
    {\{y\in \widetilde{\mathcal{B}}_{1, \infty}^t(\bar{x})\mid \forall i\in[k], j\in[d]:s_i^{(j)}(y_i^{(j)}-{\bar{x}}_i^{(j)})\geq 0\}}
     \end{multline*} 
    We remind that: 
    $$\widetilde{\mathcal{B}}_{1, \infty}^t(\bar{x})=\{y\in(\mathbb{R}^{d})^k \mid \sum_{i=1}^k \max_{j\in[d]}\delta_{\bar{x}}^{i,j}(y)\leq t\}$$
    and
    \begin{align*}
        \delta_{\bar{x}}^{i,j}(y)=\frac{y_i^{(j)}-\bar{x}_i^{(j)}}{\mathbf{1}_{y_i^{(j)}>\bar{x}_i^{(j)}}(b_i^{(j)}-\bar{x}_i^{(j)})+\mathbf{1}_{y_i^{(j)}<\bar{x}_i^{(j)}}(a_i^{(j)}-\bar{x}_i^{(j)})}
    \end{align*}
    Let
    $\Delta_{k,d,t}\triangleq\{z\in([0,\infty)^{ d})^{k}\mid\sum_{i=1}^k\max_{j\in[d]}z_i^{(j)}\leq t\}$.
    We show it is compact as a finite union of compact sets.
   To this end,
    we decompose it into sets based on the channel $j_i\in [d]$ of each entry $i\in [k]$ with the maximum value:
    \begin{multline*}
        \Delta_{k,d,t}=\bigcup_{j_1,\dots,j_k\in[d]}\{z\in([0,\infty)^{d})^{k}\mid \\ \sum_{i=1}^kz_i^{(j_i)}\leq t \wedge \forall i\in[k], j\in[d]:z_i^{(j_i)}\geq z_i^{(j)} \}
    \end{multline*}
    Note that each of these sets is bounded by linear constraints, and is thus compact. Thus, $\Delta_{k,d,t}$ is compact.
    Next, we show that $\Delta_{k,d,t}$ can be transformed to $\widetilde{\mathcal{B}}_{1, \infty}^t(\bar{x})_{s_1,\dots,s_k}$
    with an affine transformation (which in particular is continuous) and thus  $\widetilde{\mathcal{B}}_{1, \infty}^t(\bar{x})_{s_1,\dots,s_k}$ is compact as a continuous image of a compact set.  
    The affine transformation is:  
    $y_i^{(j)}=(b_i^{(j)}-{\bar{x}}_i^{(j)})\cdot z_i^{(j)} + {\bar{x}}_i^{(j)}$, if $s_i^{(j)}=1$, and $y_i^{(j)}=(a_i^{(j)}-{\bar{x}}_i^{(j)})\cdot z_i^{(j)} + {\bar{x}}_i^{(j)}$, if $s_i^{(j)}=-1$.
\end{proof}

\begin{lemma}\label{lem:4}
$E_\cap \subseteq \mathcal{B}_0^t({\bar{x}})$, where $E_\cap$ is the extreme points of $\mathcal{D}\cap\widetilde{\mathcal{B}}_{1, \infty}^t(\bar{x})$. 
\end{lemma}

\begin{proof}
We show that $E_\cap$ is contained in a subset of $\mathcal{B}_0^t({\bar{x}})$, consisting of all inputs $z$ with at least $k-t$ entries $z_i={\bar{x}}_i$ and for the other entries all the channels are either $\bar{x}_i^{(j)}$ or the bounds $a_i^{(j)}$ or $b_i^{(j)}$:
\begin{multline*}
    \{z\in \prod_{i=1}^k(\prod_{j=1}^d\{a_i^{(j)},{\bar{x}}_i^{(j)},b_i^{(j)}\})\mid \\ |\{i\in[k]:z_i\neq {\bar{x}}_i\}| \leq t\} 
\end{multline*}
  Let $y\in E_\cap$. We show:
  \begin{enumerate}[nosep,nolistsep,label=\arabic*.]
    \item $y\in \prod_{i=1}^k(\prod_{j=1}^d\{a_i^{(j)},{\bar{x}}_i^{(j)},b_i^{(j)}\})$ and
    \item if $y\in \prod_{i=1}^k(\prod_{j=1}^d\{a_i^{(j)},{\bar{x}}_i^{(j)},b_i^{(j)}\})$, then: \\
     $|\{i\in[k]\mid y_i\neq {\bar{x}}_i\}| \leq t$.
  \end{enumerate} 
   Proof for 2:
Since $y\in E_\cap$ and $E_\cap\subseteq \widetilde{\mathcal{B}}_{1,\infty}^t({\bar{x}})$, we get 
     $\sum_{i=1}^k\max_{j\in[d]}\delta_{\bar{x}}^{i,j}(y)\leq t$. We are given $y\in \prod_{i=1}^k(\prod_{j=1}^d\{a_i^{(j)},{\bar{x}}_i^{(j)},b_i^{(j)}\})$. Thus, for every $i$, either $y_i={\bar{x}}_i$ and thus $\max_{j\in[d]}\delta_{\bar{x}}^{i,j}(y)=0$ or  
     $\max_{j\in[d]}\delta_{\bar{x}}^{i,j}(y)=1$. Thus, $|\{i\in[k]\mid y_i\neq {\bar{x}}_i\}| \leq t$.\\
   Proof for 1: Assume to the contrary that there are $l\in[k]$ and $l'\in[d]$ such that $y_{l}^{(l')}\notin \{a_l^{(l')},{\bar{x}}_l^{(l')},b_l^{(l')}\}$. 
   Since $y\in E_\cap\subseteq \mathcal{D}$, we have $\delta_{\bar{x}}^{i,j}(y)\in[0,1]$ for every $i\in[k]$ and $j\in[d]$. Since $y_{l}^{(l')}\notin \{a_l^{(l')},{\bar{x}}_l^{(l')},b_l^{(l')}\}$, we get $\delta_{\bar{x}}^{l,l'}(y)\in(0,1)$.
     We show that in this case $y$ is a convex combination of two different inputs $y',y'' \in \mathcal{D}\cap\widetilde{\mathcal{B}}_{1,\infty}^t({\bar{x}})$
     and thus $y$ cannot be an extreme point, in contradiction. 
     We define these two points symmetrically, i.e., $y=0.5(y'+y'')$. 
     We remind that since $y\in E_\cap\subseteq \widetilde{\mathcal{B}}_{1,\infty}^t({\bar{x}})$, we have
     $\sum_{i=1}^k\max_{j\in[d]}\delta_{\bar{x}}^{i,j},(y)\leq t$. Accordingly, we split to cases: 
    \begin{itemize}[nosep,nolistsep]
        \item $\sum_{i=1}^k\max_{j\in[d]}\delta_{\bar{x}}^{i,j}(y)<t$: 
        In this case, $y$ is within $\widetilde{\mathcal{B}}_{1,\infty}^t({\bar{x}})$ (not on its boundary). The proof for this case is similar to the proof of the corresponding case in \Cref{lem:2}, with an additional index. We repeat it here for convenience. We identify a small value that can be added and subtracted from $y_{l}^{(l')}$ such that the new inputs do not cross ${\bar{x}}_l^{(l')}$ (possible since $y_{l}^{(l')}\neq {\bar{x}}_{l}^{(l')}$ and thus $\delta_{\bar{x}}^{l,l'}(y)>0$), do not cross the bound that is closer to $y_{l}^{(l')}$ than to ${\bar{x}}_{l}^{(l')}$ (possible since $y_{l}^{(l')}\neq a_l^{(l')},b_l^{(l')}$ and thus $\delta_{\bar{x}}^{l,l'}(y)<1$) and do not cross the boundary of $\widetilde{\mathcal{B}}_{1,\infty}^t({\bar{x}})$ (possible since $\sum_{i=1}^k\max_{j\in[d]}\delta_{\bar{x}}^{i,j}(y)<t$).
        Formally, $y'=y+\epsilon\cdot v_{l}^{(l')}$ and $y''=y-\epsilon \cdot v_{l}^{(l')}$ for:
        \begin{align*}
           v_l^{(l')}=&(\mathbf{1}_{y_{l}^{(l')}>{\bar{x}}_{l}^{(l')}}(b_{l}^{(l')}-{\bar{x}}_{l}^{(l')})\\&+\mathbf{1}_{y_{l}^{(l')}<{\bar{x}}_{l}^{(l')}}(a_{l}^{(l')}-{\bar{x}}_{l}^{(l')}))e_l^{(l')} 
        \end{align*}
        where $e_{l}^{(l')}$ is the standard unit vector with $1$ at the $(l')^{\text{th}}$ channel of entry $l$ and $0$ elsewhere, and:
        \begin{align*}
        \epsilon =\min\{\delta_{\bar{x}}^{l,l'}(y),1-\delta_{\bar{x}}^{l,l'}(y),t-\sum_{i=1}^k\max_{j\in[d]}\delta_{\bar{x}}^{i,j}(y)\}
        \end{align*}
        As $\epsilon>0$ and $v_{l}^{(l')}$ is not the zero vector (since $y_l^{(l')}$ lies strictly between $\bar{x}_l^{(l')}$ and one of the bounds $a_l^{(l')}$ or $b_l^{(l')}$) we have $y'\neq y''$ and by the construction $y',y''\in \mathcal{D}\cap\widetilde{\mathcal{B}}_{1,\infty}^t({\bar{x}})$.
        
        \item $\sum_{i=1}^k\max_{j\in[d]}\delta_{\bar{x}}^{i,j}(y)=t$: 
        In this case, $y$ is on the boundary of $\widetilde{\mathcal{B}}_{1,\infty}^t({\bar{x}})$. Similarly to the respective case in \Cref{lem:2}, we look for another entry and identify small values that can be added and subtracted in a scaled complementary manner to obtain $y'$ and $y''$. Compared to \Cref{lem:2}, this case is more challenging, since $\widetilde{\mathcal{B}}_{1,\infty}^t(\bar{x})$ considers the maximal 
        $\delta_{\bar{x}}^{i,j}(y)$ over all $d$ channels of the $i^\text{th}$ entry. Thus, we may need to add or subtract from multiple channels of the same entry. Since $y$ is an extreme point (by our assumption), for all $i \in [k]$, the non zero values among $\delta_{\bar{x}}^{i,j}(y)$ for $j \in [d]$ are equal (the proof is given at the end of this proof and marked with $\star$).  
        Given $i\in [k]$, we denote by  $\mathcal{J}_i$ the set of all channels $j\in[d]$ such that $\delta_{\bar{x}}^{i,j}(y)\neq 0$.   
        Next, we consider $i=l$.
        Since $\delta_{\bar{x}}^{l,l'}(y)\in(0,1)$, we have $\delta_{\bar{x}}^{l,j}(y)=\delta_{\bar{x}}^{l,l'}(y)\in(0,1)$ for all $j\in \mathcal{J}_l$. Thus, $y_l^{(j)}\notin\{a_l^{(j)},\bar{x}_l^{(j)},b_l^{(j)}\}$ for all $j\in \mathcal{J}_l$. 
         We show that there is another entry $m\neq l$ such that $\max_{j \in [d]} \delta_{\bar{x}}^{m,j}(y)\in(0,1)$ and $y_m^{(j)}\notin\{a_m^{(j)},\bar{x}_m^{(j)},b_m^{(j)}\}$ for all $j\in \mathcal{J}_m$
          and identify small values that can be added and subtracted from $y_{l}^{(j)},y_m^{(j')}$ (for all $j\in\mathcal{J}_l$ and $j'\in\mathcal{J}_m$) such that the new inputs do not cross ${\bar{x}}_l^{(j)},{\bar{x}}_m^{(j')}$(possible since $y_{l}^{(j)}\neq {\bar{x}}_{l}^{(j)},y_m^{(j')}\neq {\bar{x}}_m^{(j')}$ and thus $\delta_{\bar{x}}^{l,j}(y),\delta_{\bar{x}}^{m,j'}(y)>0$, for all $j\in\mathcal{J}_l$ and $j'\in\mathcal{J}_m$), do not cross the bound that is closer to $y_{l}^{(j)}$ and $y_m^{(j')}$ than to ${\bar{x}}_{l}^{(j)}$ and ${\bar{x}}_m^{(j')}$ (possible since $y_{l}^{(j)}\neq a_l^{(j)},b_l^{(j)}$ and $y_m^{(j')} \neq a_m^{(j')},b_m^{(j')}$ and thus $\delta_{\bar{x}}^{l,j}(y),\delta_{\bar{x}}^{m,j'}(y)<1$, for all $j\in\mathcal{J}_l$ and $j'\in\mathcal{J}_m$) and remain on the boundary of $\widetilde{\mathcal{B}}_{1,\infty}^t({\bar{x}})$.
        Since $\sum_{i=1}^k\max_{j\in[d]}\delta_{\bar{x}}^{i,j}(y)=t$ for an integer $t$ and since $\max_{j\in[d]}\delta_{\bar{x}}^{l,j}(y)\in(0,1)$, there must be $m\neq l$ such that $\max_{j\in [d]}\delta_{\bar{x}}^{m,j}(y)\in(0,1)$. As explained before, for all $j\in \mathcal{J}_m$, we have $y_{m}^{(j)}\notin \{a_m^{(j)},{\bar{x}}_m^{(j)},b_m^{(j)}\}$. 
         We define $y'=y+\epsilon\cdot{v_{lm}}$ and $y''=y-\epsilon\cdot {v_{lm}}$, where ${v_{lm}}$ changes coordinates in $l$ and $m$ in a scaled complementary manner to preserve $\sum_{i=1}^k\max_{j\in[d]}\delta_{\bar{x}}^{i,j}(y)=t$. Formally, $v_{lm}=v_l-v_m$ where:
        \begin{align*}
            v_l = \sum_{j\in \mathcal{J}_l}&(\mathbf{1}_{\{y_{l}^{(j)}>{\bar{x}}_{l}^{(j)}\}}(b_{l}^{(j)}-{\bar{x}}_{l}^{(j)})\\&+\mathbf{1}_{\{y_{l}^{(j)}<{\bar{x}}_{l}^{(j)}\}}(a_{l}^{(j)}-{\bar{x}}_{l}^{(j)})) \cdot e_{l}^{(j)}
        \end{align*}
        and similarly:
        \begin{align*}
            v_m = \sum_{j\in \mathcal{J}_m}&(\mathbf{1}_{\{y_{m}^{(j)}>{\bar{x}}_{m}^{(j)}\}}(b_{m}^{(j)}-{\bar{x}}_{m}^{(j)})\\&+\mathbf{1}_{\{y_{m}^{(j)}<{\bar{x}}_{m}^{(j)}\}}(a_{m}^{(j)}-{\bar{x}}_{m}^{(j)})) \cdot e_{m}^{(j)}
        \end{align*}
        and:
        \begin{align*}
        \epsilon =\min[
            &\min_{j\in \mathcal{J}_l}\{\delta_{\bar{x}}^{l,j}(y)\},\min_{j\in \mathcal{J}_l}\{1- \delta_{\bar{x}}^{l,j}(y)\},\\
            &\min_{j\in \mathcal{J}_m}\{\delta_{\bar{x}}^{m,j}(y)\},\min_{j\in \mathcal{J}_m}\{1- \delta_{\bar{x}}^{m,j}(y)\}
            ]
        \end{align*}
        As $\epsilon>0$ and $v_{lm}$ is not the zero vector, we have $y'\neq y''$ and by the construction $y',y''\in \mathcal{D}\cap\widetilde{\mathcal{B}}_{1,\infty}^t({\bar{x}})$.

        Proof for $\star:$ We prove that, if $y$ in an extreme point, then for every $i\in[k]$ all the non zero values of $\delta_{\bar{x}}^{i,j}(y)$ are equal. Assume to the contrary that
        there are $l\in[k]$ and $l'\in[d]$ such that $0<\delta_{\bar{x}}^{l,l'}(y)<\max_{j\in[d]}\delta_{\bar{x}}^{l,j}(y)$. We can slightly increase or decrease this channel without crossing $\bar{x}_l^{(l')}$ and the bound that is closer to $y_{l}^{(l')}$ than to ${\bar{x}}_{l}^{(l')}$ and without changing $\max_{j\in[d]}\delta_{\bar{x}}^{l,j}(y)$ to obtain two different inputs $y',y''\in \mathcal{D}\cap\widetilde{\mathcal{B}}_{1, \infty}^t(\bar{x})$ such that $y=0.5(y'+y'')$ and thus $y$ is not an extreme point. Formally, we define $y'=y+\epsilon\cdot v_{l}^{(l')}$ and $y''=y-\epsilon\cdot v_{l}^{(l')}$ for
        \begin{align*}
           v_l^{(l')}=&(\mathbf{1}_{y_{l}^{(l')}>{\bar{x}}_{l}^{(l')}}(b_{l}^{(l')}-{\bar{x}}_{l}^{(l')})\\&+\mathbf{1}_{y_{l}^{(l')}<{\bar{x}}_{l}^{(l')}}(a_{l}^{(l')}-{\bar{x}}_{l}^{(l')}))e_l^{(l')} 
        \end{align*}
        and
       $ \epsilon =\min\{\delta_{\bar{x}}^{l,l'}(y),\max_{j\in[d]}\delta_{\bar{x}}^{l,j}(y)-\delta_{\bar{x}}^{l,l'}(y)\}$. 
        As $\epsilon>0$ and $v_{l}^{(l')}$ is not the zero vector, we have $y'\neq y''$ and by the construction $y',y''\in \mathcal{D}\cap\widetilde{\mathcal{B}}_{1,\infty}^t({\bar{x}})$.
    \end{itemize}
\end{proof}

\begin{restatable}[]{theorem}{mvolumes}\label{thm4}
    \begin{enumerate}[label=\arabic*.]
        \item $\text{vol}(\widetilde{\mathcal{B}}_{1,\infty}^t(\bar{x}))=\text{vol}(\mathcal{D})\frac{t^{dk}(d!)^k}{(dk)!}$.
        \item $\text{vol}(\mathcal{D}\cap\widetilde{\mathcal{B}}_{1,\infty}^t(\bar{x}))=\text{vol}(\mathcal{D})\frac{t^{dk}(d!)^k}{(dk)!}(1+\sum_{r=1}^{t-1}c_{r,k,d,t})$,
    \end{enumerate}
    where \begin{multline*}
        c_{r,k,d,t}=(-1)^{r}\binom{k}{r}\\
        \sum_{m_1,\dots,m_r\in[d]}\frac{(dk)!(1-\frac{r}{t})^{dk}(t-r)^{-dr+\sum_{i=1}^rm_i}}{(d(k-r)+\sum_{i=1}^rm_i)!\prod_{i=1}^r(d-m_i)!}\\
    \end{multline*}
\end{restatable}
Like~\Cref{thm2}, given fixed $d,t$ and $r\in[1,t-1]$, we have that $c_{r,k,d,t}$ exponentially converges to $0$ as $k$ increases.
\begin{proof}
    Like \Cref{thm2}, we decompose $\widetilde{\mathcal{B}}_{1,\infty}^t(\bar{x})$ and $\mathcal{D}\cap\widetilde{\mathcal{B}}_{1,\infty}^t(\bar{x})$ by the $2^{d\cdot k}$ orthants, where each is a scaled version of $\Delta_{k,d,t}=\{z\in([0,\infty)^d)^k \mid \sum_{i=1}^k \max_{j\in [d]} z_i^{(j)}\leq t\}$ for $\widetilde{\mathcal{B}}_{1,\infty}^t(\bar{x})$ and $\Delta_{k,d,t}\cap ([0,1]^d)^k$ for $\mathcal{D}\cap\widetilde{\mathcal{B}}_{1,\infty}^t(\bar{x})$.
    Like \Cref{thm2}, the summation of these scalings is $\text{vol}(\mathcal{D})$, thus
    $\text{vol}(\widetilde{\mathcal{B}}_{1,\infty}^t(\bar{x}))=\text{vol}(\mathcal{D})\cdot \text{vol}(\Delta_{k,d,t})$ and 
    $\text{vol}(\widetilde{\mathcal{B}}_{1,\infty}^t(\bar{x})\cap\mathcal{D})=\text{vol}(\mathcal{D})\cdot \text{vol}(\Delta_{k,d,t}\cap([0,1]^d)^k)$.
    
    Next, we compute $\text{vol}(\Delta_{k,d,t})$: 
    \begin{multline*}
        \text{vol}(\Delta_{k,d,t}) = \int_{\substack{\sum_{i=1}^k\max_{j\in [d]} z_i^{(j)}\leq t\\
        \forall i \in [k],\, j \in [d]:\, 0\leq z_i^{(j)}}}
        \prod_{i=1}^k\prod_{j=1}^d\mathrm{d}z_i^{(j)}\\
        =d^k\int_{\substack{\sum_{i=1}^k z_i^{(1)} \leq t \\
             \forall i \in [k],\, j \in [d]:\, 0 \leq z_i^{(j)}\leq z_i^{(1)}} }
             \prod_{i=1}^k\prod_{j=1}^d\mathrm{d}z_i^{(j)}\\
        =d^k\int_{\substack{\sum_{i=1}^k z_i^{(1)}\leq t \\ \forall i \in [k]:\,0\leq z_i^{(1)}}} \prod_{i=1}^k\mathrm{d}z_{i}^{(1)}\prod_{j=2}^d\int_{0}^{z_i^{(1)}}\mathrm{d}z_i^{(j)}\\
        =d^k\int_{\substack{\sum_{i=1}^k z_i^{(1)}\leq t \\ \forall i \in [k]:\,0 \leq z_i^{(1)}}}
             \prod_{i=1}^k(z_i^{(1)})^{d-1}\mathrm{d}z_{i}^{(1)}\\
        =d^k\frac{t^{dk}((d-1)!)^k}{(dk)!}=\frac{t^{dk}(d!)^k}{(dk)!}
    \end{multline*}
    The second equality follows by partitioning $\Delta_{k,d,t}$ into $d^k$ symmetric options (one for each possibility of a maximal channel $j\in [d]$ for each $i\in [k]$). Note that the intersection of any two options has a zero volume.
     The third equality follows by Fubini's theorem. 
     The fourth equality follows by computing the inner integrals, each contributing $z_i^{(1)}$.
     The fifth equality follows by \citet[Eq.~4.634]{gradshteyn}.
     The last transition unifies $d^k$ with $(d-1)!^k$.
  
    We continue with $\text{vol}(\Delta_{k,d,t}\cap([0,1]^d)^k)$, which is identical to the first four transitions above, except that $z_i^{(j)}\leq 1$. We continue from the fourth transition, followed by a transition that omits the superscript $^{(1)}$ to simplify notations for the following equalities:
    \begin{multline*}
    \text{vol}(\Delta_{k,d,t}\cap([0,1]^d)^k)\\
        =d^k\int_{\substack{\sum_{i=1}^k z_i^{(1)}\leq t \\ \forall i \in [k]:\, 0\leq z_i^{(1)}\leq 1}}
             \prod_{i=1}^k(z_i^{(1)})^{d-1}\mathrm{d}z_{i}^{(1)}\\
        =d^k\int_{\substack{\sum_{i=1}^k z_i\leq t \\ \forall i \in [k]:\, 0\leq z_i\leq 1}}
             \prod_{i=1}^kz_i^{d-1}\mathrm{d}z_{i}
    \end{multline*}
     To compute this integral, we remove the requirement of $z_i\leq 1$ and then repair it by excluding all the cases of $z_i\geq 1$, using the inclusion-exclusion principle (similarly to a geometric derivation of Irwin-Hall by \citet{Geometric-Irwin-Hall}). 
    \begin{multline*}
    =d^k\int_{\substack{\sum_{i=1}^k z_i\leq t \\ \forall i \in [k]:\, 0 \leq z_i}}
             \prod_{i=1}^kz_i^{d-1}\mathrm{d}z_{i} \\ 
     +d^k\sum_{r=1}^{k}(-1)^{r}\sum_{1\leq i_1<\dots<i_r\leq k}\int_{\substack{\sum_{i=1}^k z_i\leq t \\ \forall i \in [k]:\, 0 \leq z_i \\ 1\leq z_{i_1},\dots,z_{i_r}}}\prod_{i=1}^kz_i^{d-1}\mathrm{d}z_{i}\\
     =\frac{t^{dk}(d!)^k}{(dk)!}(1+\sum_{r=1}^{k}c_{r,k,d,t})
    \end{multline*}
    where the last transition extracts the left term and denotes the rest of the terms $c_{r,k,d,t}$, where:
    \begin{multline*}
        c_{r,k,d,t}=\frac{(-1)^{r}(dk)!}{t^{dk}((d-1)!)^k}\\
        \sum_{1\leq i_1<\dots<i_r\leq k}\int_{\substack{\sum_{i=1}^k z_i\leq t \\ \forall i \in [k]:\, 0 \leq z_i \\ 1\leq z_{i_1},\dots,z_{i_r}}}\prod_{i=1}^kz_i^{d-1}\mathrm{d}z_{i}\\
        =\frac{(-1)^{r}(dk)!}{t^{dk}((d-1)!)^k}
        \binom{k}{r}\int_{\substack{\sum_{i=1}^k z_i\leq t \\ \forall i \in [k]:\, 0 \leq z_i \\ 1\leq z_{1},\dots,z_{r}}}\prod_{i=1}^kz_i^{d-1}\mathrm{d}z_{i}\\
    \end{multline*}
    The above transition follows since the integral is the same for every choice of $i_1,\dots,i_r$. 
    Note that if $r\geq t$, then $c_{r,k,d,t}=0$, because of the conditions $\sum_{i=1}^k z_i\leq t $ and $1\leq z_{1},\dots,z_{r}$. Thus, the sum is up to $t-1$ instead of $k$:
    \begin{multline*}
    \text{vol}(\Delta_{k,d,t}\cap([0,1]^d)^k)
        =\frac{t^{dk}(d!)^k}{(dk)!}(1+\sum_{r=1}^{t-1}c_{r,k,d,t})
    \end{multline*}

    Lastly, we wish to compute the integral of $c_{r,k,d,t}$ using ~\citet[Eq.~4.634]{gradshteyn}. To this end, we change variables and rely on the binomial expansion.
    We first change the variables $z_i\mapsto z_i -1$ for $i\in[r]$:
    \begin{multline*}
    \int_{\substack{\sum_{i=1}^k z_i\leq t \\ \forall i \in [k]:\, 0\leq z_i \\ 1\leq z_{1},\dots,z_{r}}}\prod_{i=1}^kz_i^{d-1}\mathrm{d}z_{i}\\
        =\int_{\substack{\sum_{i=1}^k z_i\leq t-r \\ \forall i \in [k]:\, 0\leq z_i}}\left(\prod_{i=1}^r(z_i+1)^{d-1}\mathrm{d}z_{i}\right)\left(\prod_{i=r+1}^kz_i^{d-1}\mathrm{d}z_{i}\right)
    \end{multline*}
    We then rely on the binomial expansion for $\prod_{i=1}^r(z_i+1)^{d-1}$:
    \begin{multline*}
        \prod_{i=1}^r(z_i+1)^{d-1}
        =\prod_{i=1}^r\sum_{m=0}^{d-1}\binom{d-1}{m}z_i^{m}=\\
        \prod_{i=1}^r\sum_{m=1}^{d}\binom{d-1}{m-1}z_i^{m-1}
        =\sum_{m_1,\dots,m_r\in[d]}\prod_{i=1}^r\binom{d-1}{m_i-1}z_i^{m_i-1}
    \end{multline*}
    The second transition follows by changing the sum over $1$ to $d$. The third transition switches the multiplication and the sum.
    Substituting it into the integral:
    \begin{multline*}
    \int_{\substack{\sum_{i=1}^k z_i\leq t \\ \forall i \in [k]:\, 0\leq z_i \\ 1\leq z_{1},\dots,z_{r}}}\prod_{i=1}^kz_i^{d-1}\mathrm{d}z_{i}\\
        =\sum_{m_1,\dots,m_r\in[d]}\left(\prod_{i=1}^r\binom{d-1}{m_i-1}\right)\\
        \int_{\substack{\sum_{i=1}^k z_i\leq t-r \\ \forall i \in [k]:\, 0\leq z_i}}\prod_{i=1}^rz_i^{m_i-1}\mathrm{d}z_{i}\prod_{i=r+1}^kz_i^{d-1}\mathrm{d}z_{i}\\
        =\sum_{m_1,\dots,m_r\in[d]}\left(\prod_{i=1}^r\binom{d-1}{m_i-1}\right)\\
        \frac{(t-r)^{d(k-r)+\sum_{i=1}^rm_i}((d-1)!)^{k-r}\prod_{i=1}^r(m_i-1)!}{(d(k-r)+\sum_{i=1}^rm_i)!}\\
        =\sum_{m_1,\dots,m_r\in[d]}\frac{(t-r)^{d(k-r)+\sum_{i=1}^rm_i}((d-1)!)^k}{(d(k-r)+\sum_{i=1}^rm_i)!\prod_{i=1}^r(d-m_i)!}\\
    \end{multline*}
The second equality follows by~\citet[Eq.~4.634]{gradshteyn}. The last equality is algebraic rewrites.
   We substitute into $c_{r,k,d,t}$ and simplify with algebraic rewrite:
    \begin{multline*}
        c_{r,k,d,t}=\frac{(-1)^{r}(dk)!}{t^{dk}((d-1)!)^k}\binom{k}{r}\\
        \sum_{m_1,\dots,m_r\in[d]}\frac{(t-r)^{d(k-r)+\sum_{i=1}^rm_i}((d-1)!)^k}{(d(k-r)+\sum_{i=1}^rm_i)!\prod_{i=1}^r(d-m_i)!}\\
        =(-1)^{r}\binom{k}{r}\\
        \sum_{m_1,\dots,m_r\in[d]}\frac{(dk)!(1-\frac{r}{t})^{dk}(t-r)^{-dr+\sum_{i=1}^rm_i}}{(d(k-r)+\sum_{i=1}^rm_i)!\prod_{i=1}^r(d-m_i)!}\\
    \end{multline*}
    Note that $(1-\frac{r}{t})^{dk}$ causes the exponential decay of $ c_{r,k,d,t}$, since  $\frac{(dk)!}{(d(k-r)+\sum_{i=1}^rm_i)!}$ and $\binom{k}{r}$ are polynomial in $k$ and the rest are independent of $k$.
\end{proof}

\section{Implementation}
In this section, we describe our extension to GPUPoly to support our bound propagations.
We focus on single-channel inputs. The extension to multi-channel inputs is very similar.
 
Recall that GPUPoly computes bounds over the box domain, where the lower bound is $l=\min_{y\in \mathcal{D}}\sum_{i=1}^kw_iy_i=\sum_{i=1}^kw_i\bar{x}_i+\sum_{i=1}^kd^-_{i}$ (the upper bound is analogous). The CUDA kernel
  distributes the set of indices $[k]$ across threads.
Each thread computes, for every assigned index $i$, the value $\min(w_i a_i, w_i b_i) = w_i \bar{x}_i + d_i^{-}$ and sums these values over its assigned indices. These sums from all threads are then combined using a tree-based reduction scheme to obtain the lower bound. The GPU's shared memory is used to share the intermediate sums between threads.

For $\widetilde{\mathcal{B}}_1^t(x)$, we remind that the lower bound is $l=\min_{y\in \widetilde{\mathcal{B}}_1^t(\bar{x})}\sum_{i=1}^kw_iy_i=\left(\sum_{i=1}^kw_i\bar{x}_i\right)+t\cdot d^-_{l_1}$.
To efficiently compute the sum $w_i \bar{x}_i$ and add $t$ times the smallest $d_i^{-}$, we modify the CUDA kernel so that each thread computes the sum $w_i \bar{x}_i$ over its assigned indices and records the minimum $d_i^{-}$ over its assigned indices. In the tree-based reduction scheme, it combines the sums and propagates the minimum $d_i^{-}$ observed so far. Finally, the lower bound is obtained by adding $t$ times the minimum $d_i^{-}$ to the accumulated sum. Synchronizing the minimum $d_i^{-}$ across threads doubles the required shared memory in the reduction scheme.

For the $\ell_0$-ball, the lower bound is $l=\min_{y\in \mathcal{B}_0^t(\bar{x})}\sum_{i=1}^kw_iy_i=\left(\sum_{i=1}^kw_i\bar{x}_i\right)+(d^-_{l_1}+\dots+ d^-_{l_t})$. 
As in the previous case, each thread records the $t$ smallest $d_i^{-}$ values it encounters and selects the values to propagate at each reduction step.
This computation increases the complexity of each synchronization step to $\mathcal{O}(t)$ (instead of $\mathcal{O}(1)$).
 If we kept double precision, this would increase the shared memory usage by a factor of $t+1$.
Instead, our implementation propagates the $d_i^{-}$ values as floats (instead of doubles).

We note that, besides the aforementioned changes, our implementation does not modify GPUPoly, thus benefiting from its optimizations. 
Additionally, the resource requirements remain comparable to those of GPUPoly.

\begin{figure}[t]
\centering
\includegraphics[width=0.65\columnwidth]{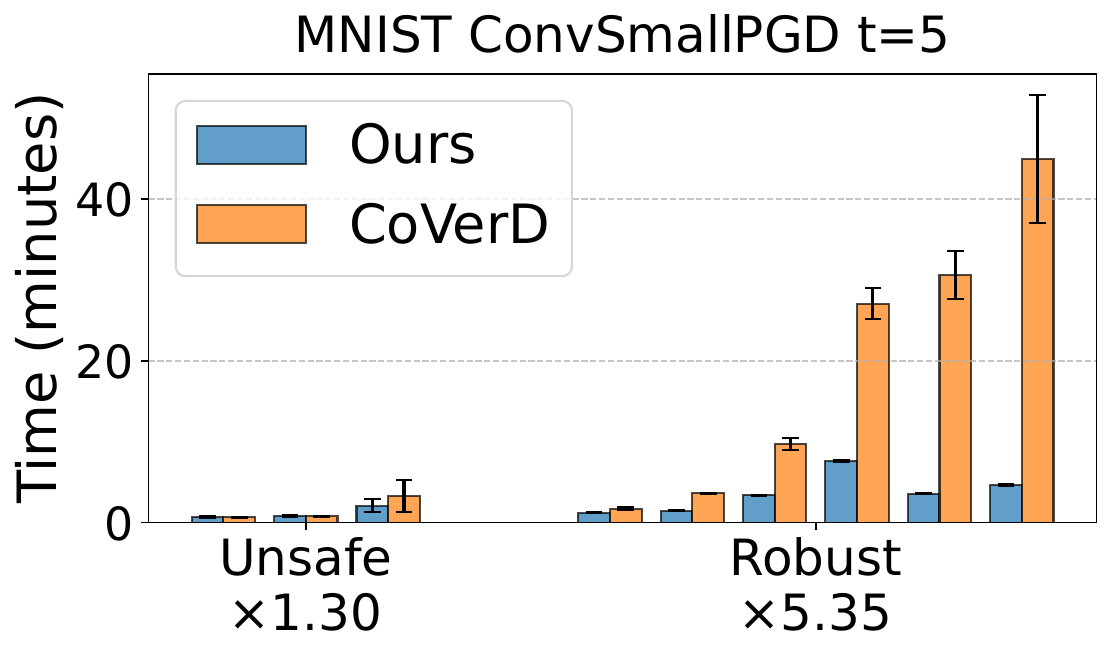}
\caption{{\tool} + \topt compared to \tool on a less challenging benchmark.}
\label{fig:more_comp_5}
\end{figure}

\begin{figure}[t]
\centering
\includegraphics[width=0.95\columnwidth]{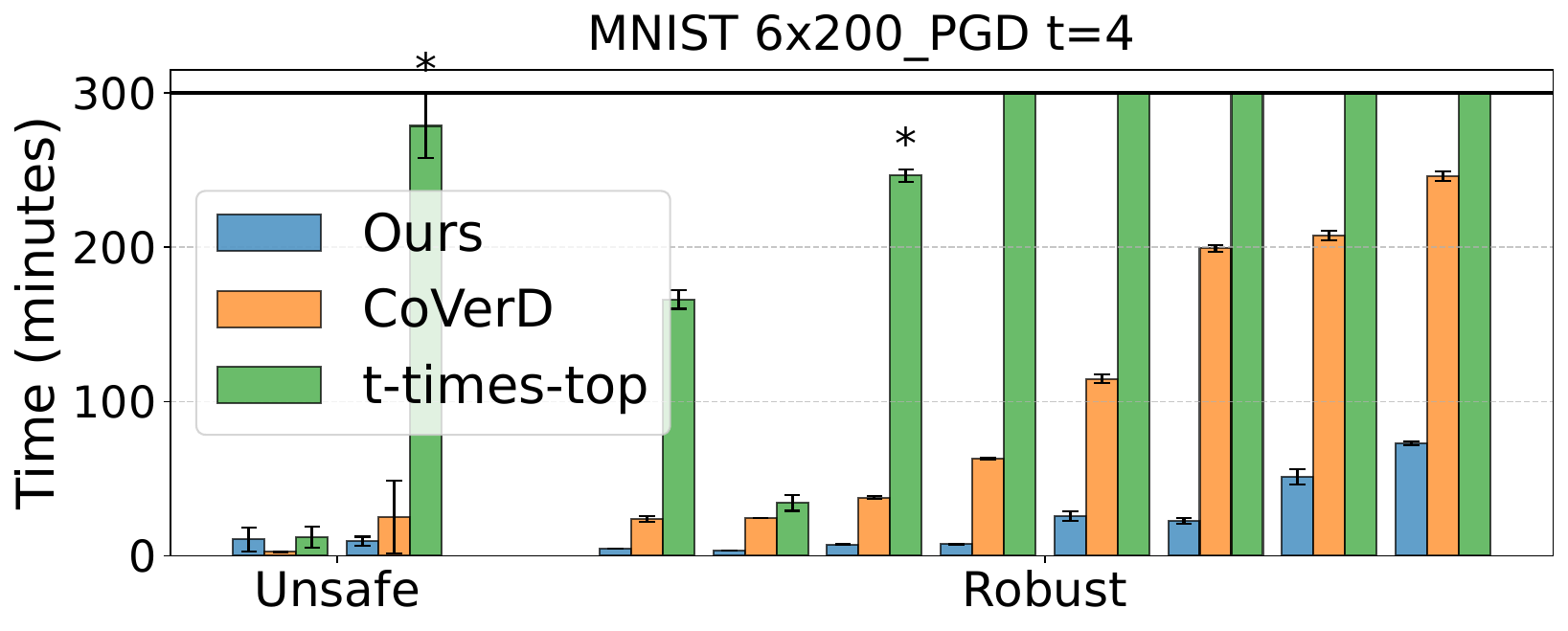}
\caption{\tool with \topo.}
\label{fig:coverd+t-times-top}
\end{figure}

\section{Additional Results}
In this section, we provide additional results.

\begin{figure*}[t]
\centering
\includegraphics[width=0.82\textwidth]{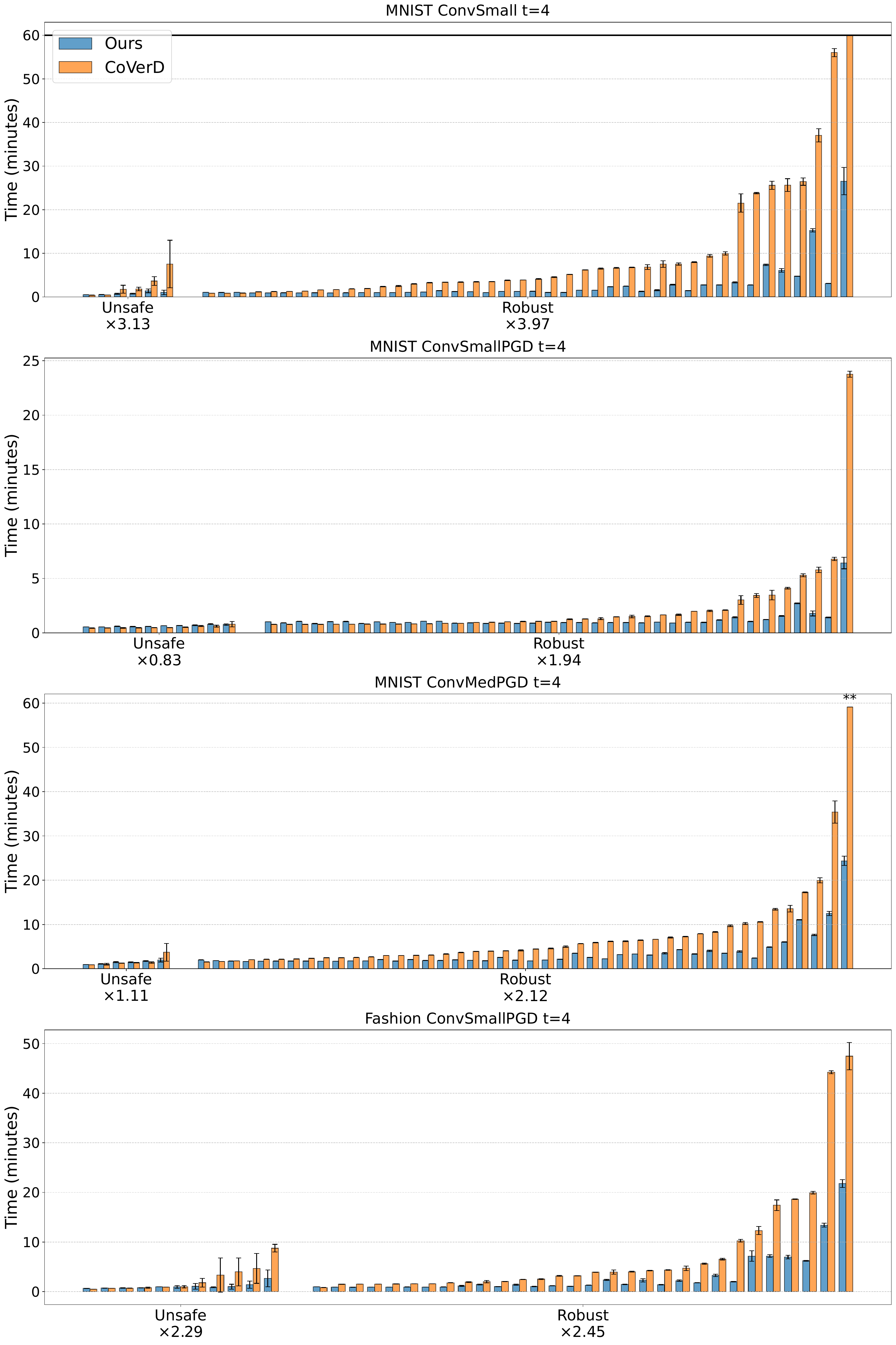}
\caption{{\tool} + \topt compared to \tool on less challenging benchmarks.}
\label{fig:more_comp_4}
\end{figure*} 

\Cref{fig:more_comp_5} and \Cref{fig:more_comp_4} extend \Cref{fig:covev} with \tool's less challenging benchmarks (smaller values of $t$ than in~\Cref{fig:covev}).
For a single $\ell_0$-ball, two runs of \tool out of three reached the timeout (denoted by **) and as in figure \Cref{fig:covev}, these runs are excluded. The figures show that our approach significantly boosts \tool.

\Cref{fig:coverd+t-times-top} compares \tool+\topo with our approach and \tool on the MNIST 6x200\_PGD network. 
For two $\ell_0$-balls, one run of \tool+\topo out of three reached the timeout (denoted by *). 
It shows that not only our approach is significantly better, but also \tool is better than \tool+\topo, which is consistent with~\Cref{fig:bound_prop}.

\begin{figure*}[t]
\centering
\includegraphics[width=0.95\textwidth]{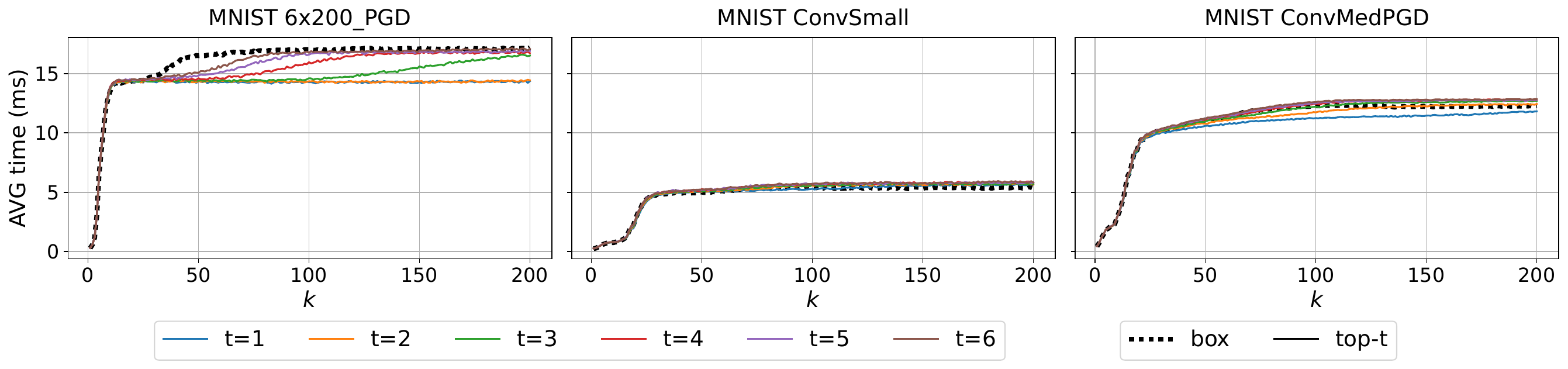}
\caption{The mean running time of GPUPoly and \topt on neighborhoods capturing perturbations to subsets of $k$ pixels.}
\label{fig:bound_prop_times}
\end{figure*}

\begin{table*}[t]
\centering
{\fontsize{9}{11}\selectfont
\setlength{\tabcolsep}{1mm}
\begin{tabular}{llllccccccccccc}
\toprule
\multicolumn{4}{c}{} & \multicolumn{10}{c}{\textbf{Image Index}} \\
\cmidrule(lr){6-15}
\textbf{Dataset} & \textbf{Network} & \textbf{$t$} & \textbf{Seed} & & \textbf{0} & \textbf{1} & \textbf{2} & \textbf{3} & \textbf{4} & \textbf{5} & \textbf{6} & \textbf{7} & \textbf{8} & \textbf{9} \\
\midrule

MNIST & $6\times200$\_PGD & $4$ & 42 & \tool & 24.3 & 36.5 & 205.7 & 23.8 & 3.0* & 62.1 & 112.9 & 246.8 & 198.0 & 12.8* \\
      &                   &     &   & Ours  & 3.4 & 6.8 & 53.5 & 4.8 & 2.2* & 7.8 & 22.0 & 74.4 & 25.5 & 13.5* \\
\cmidrule(lr){4-15}
      &                   &     & 41 & \tool & 24.5 & 38.9 & 212.1 & 21.5 & 2.7* & 63.6 & 112.5 & 241.9 & 202.5 & 4.6* \\
      &                   &     &   & Ours  & 3.5 & 7.6 & 44.3 & 4.8 & 9.0* & 7.7 & 29.8 & 72.2 & 21.4 & 6.7* \\
\cmidrule(lr){4-15}
      &                   &     & 40 & \tool & 24.5 & 37.7 & 205.6 & 26.3 & 2.2* & 63.6 & 118.6 & 249.5 & 197.8 & 58.4* \\
      &                   &     &   & Ours  & 3.4 & 7.5 & 56.2 & 4.9 & 20.7* & 7.3 & 26.1 & 71.7 & 21.6 & 8.3* \\
\cmidrule(lr){2-15}
      & ConvSmall & $5$ & 42 & \tool & 15.1 & 29.4 & TO & 3.8 & 0.9* & TO & TO & 104.7 & - & 104.0 \\
      &           &     &   & Ours  & 3.5 & 3.9 & 72.7 & 1.2 & 0.7* & 128.0 & 117.9 & 12.2 & - & 9.3 \\
\cmidrule(lr){4-15}
      &           &     & 41 & \tool & 15.3 & 30.2 & TO & 4.0 & 0.7* & TO & TO & 87.0 & - & 62.5 \\
      &           &     &   & Ours  & 3.4 & 3.9 & 72.8 & 1.2 & 0.8* & 102.3 & 114.8 & 12.0 & - & 9.5 \\
\cmidrule(lr){4-15}
      &           &     & 40 & \tool & 15.0 & 31.9 & TO & 3.8 & 1.5* & TO & TO & 69.9 & - & 62.1 \\
      &           &     &   & Ours  & 3.5 & 4.0 & 72.3 & 1.2 & 0.9* & 135.8 & 148.4 & 12.1 & - & 9.4 \\
\cmidrule(lr){2-15}
      & ConvSmallPGD & $6$ & 42 & \tool & 47.3 & TO & TO & 27.1 & 4.0* & 213.0 & TO & 2.2* & - & 2.2* \\
      &              &     &   & Ours  & 9.9 & 59.4 & 193.5 & 3.0 & 2.6* & 42.0 & 116.0 & 2.5* & - & 2.6* \\
\cmidrule(lr){4-15}
      &              &     & 41 & \tool & 50.3 & TO & TO & 65.6 & 1.5* & 226.4 & TO & 2.3* & - & 2.0* \\
      &              &     &   & Ours  & 10.5 & 104.7 & 154.8 & 2.9 & 2.0* & 38.7 & 105.5 & 2.4* & - & 2.3* \\
\cmidrule(lr){4-15}
      &              &     & 40 & \tool & 61.9 & TO & TO & 14.6 & 1.6* & TO & TO & 2.4* & - & 2.1* \\
      &              &     &   & Ours  & 9.3 & 104.6 & 163.1 & 2.9 & 4.7* & 38.1 & 201.2 & 2.4* & - & 2.4* \\
\cmidrule(l){2-15}
      & ConvMedPGD & $5$ & 42 & \tool & 37.5 & 91.0 & 111.8 & 8.1 & 1.2* & 38.7 & 136.4 & 213.4 & 1.2* & 100.2 \\
      &            &     &   & Ours  & 14.7 & 15.9 & 58.5 & 2.4 & 1.4* & 17.6 & 47.9 & 62.4 & 1.4* & 21.6 \\
\cmidrule(lr){4-15}
      &            &     & 41 & \tool & 38.7 & 87.0 & 111.3 & 7.8 & 1.2* & 37.8 & 136.5 & 215.6 & 1.3* & 97.8 \\
      &            &     &   & Ours  & 14.2 & 15.8 & 60.2 & 2.3 & 1.4* & 17.6 & 46.2 & 63.0 & 1.3* & 27.3 \\
\cmidrule(lr){4-15}
      &            &     & 40 & \tool & 37.6 & 86.8 & 111.2 & 7.8 & 1.2* & 38.2 & 137.0 & 240.8 & 1.3* & 99.9 \\
      &            &     &   & Ours  & 14.0 & 19.1 & 59.8 & 2.3 & 1.5* & 17.5 & 57.5 & 64.0 & 1.3* & 20.9 \\
\cmidrule(lr){2-15}
      & ConvBig & $3$ & 42 & \tool & 22.9 & 20.8 & 252.0 & 10.5 & 36.7 & 90.8 & 30.4 & 18.7 & - & 13.4 \\
      &         &     &   & Ours  & 18.6 & 14.5 & 216.3 & 14.0 & 19.6 & 72.2 & 17.8 & 13.1 & - & 11.2 \\
\cmidrule(lr){4-15}
      &         &     & 41 & \tool & 22.9 & 20.5 & 250.8 & 10.5 & 36.7 & 90.4 & 31.0 & 18.3 & - & 13.3 \\
      &         &     &   & Ours  & 18.5 & 14.4 & 218.0 & 14.1 & 20.3 & 71.2 & 17.6 & 13.2 & - & 11.3 \\
\cmidrule(lr){4-15}
      &         &     & 40 & \tool & 23.0 & 20.7 & 250.1 & 10.4 & 37.0 & 90.3 & 30.8 & 18.3 & - & 13.3 \\
      &         &     &   & Ours  & 18.7 & 14.4 & 216.6 & 14.1 & 19.6 & 71.7 & 17.9 & 13.2 & - & 11.4 \\

\bottomrule
\end{tabular}
}
\caption{A detailed list of our experimental results on the MNIST dataset.}
\label{table:comparison_to_coverd}
\end{table*}

\begin{table*}[t]
\centering
{\fontsize{9}{11}\selectfont
\setlength{\tabcolsep}{1mm} 
\begin{tabular}{llllccccccccccc}
\toprule
\multicolumn{4}{c}{} & \multicolumn{10}{c}{\textbf{Image Index}} \\
\cmidrule(lr){6-15}
\textbf{Dataset} & \textbf{Network} & \textbf{$t$} & \textbf{Seed} & & \textbf{0} & \textbf{1} & \textbf{2} & \textbf{3} & \textbf{4} & \textbf{5} & \textbf{6} & \textbf{7} & \textbf{8} & \textbf{9} \\
\midrule

Fashion & ConvSmallPGD & $5$ & 42 & \tool & 19.6 & 69.5 & TO & - & 15.0* & - & 21.6 & 0.8* & 22.0 & 25.3 \\
        &              &     &   & Ours  & 3.5 & 30.8 & TO & - & 18.6* & - & 5.0 & 0.9* & 4.7 & 5.5 \\
\cmidrule(lr){4-15}
        &              &     & 41 & \tool & 20.9 & 81.8 & TO & - & 96.8* & - & 21.1 & 0.8* & 20.7 & 25.0 \\
        &              &     &   & Ours  & 3.4 & 37.2 & TO & - & 22.7* & - & 4.9 & 1.0* & 4.1 & 5.5 \\
\cmidrule(lr){4-15}
        &              &     & 40 & \tool & 20.3 & 84.2 & TO & - & 134.5* & - & 20.9 & 0.8* & 21.6 & 24.7 \\
        &              &     &   & Ours  & 3.4 & 34.5 & TO & - & 5.9* & - & 5.0 & 0.9* & 4.1 & 5.5 \\
\cmidrule(lr){2-15}
        & ConvMedPGD & $4$ & 42 & \tool & 17.5 & 8.8 & 98.4 & 3.1* & 54.4 & - & 6.3 & - & 16.4 & 17.4 \\
        &            &     &   & Ours  & 4.4 & 2.8 & 40.7 & 3.9* & 16.6 & - & 1.9 & - & 4.8 & 5.5 \\
\cmidrule(lr){4-15}
        &            &     & 41 & \tool & 17.3 & 8.9 & 98.4 & 2.5* & 60.5 & - & 6.5 & - & 16.3 & 17.7 \\
        &            &     &   & Ours  & 4.4 & 2.8 & 40.1 & 1.2* & 15.2 & - & 1.9 & - & 4.6 & 5.5 \\
\cmidrule(lr){4-15}
        &            &     & 40 & \tool & 17.2 & 8.8 & 99.1 & 2.9* & 58.1 & - & 6.4 & - & 16.1 & 18.1 \\
        &            &     &   & Ours  & 4.4 & 2.8 & 40.3 & 3.9* & 14.8 & - & 1.9 & - & 4.8 & 5.6 \\

\midrule
CIFAR-10 & ConvSmallPGD & $3$ & 42 & \tool & 41.4 & 42.0 & 50.3 & 3.5* & 5.7* & 184.1 & 0.6* & 224.9 & 5.3* & - \\
         &              &     &   & Ours  & 28.3 & 26.3 & 28.5 & 3.9* & 5.1* & 131.5 & 0.5* & 157.8 & 7.8* & - \\
\cmidrule(lr){4-15}
         &              &     & 41 & \tool & 41.7 & 41.9 & 46.9 & 5.7* & 3.5* & 180.6 & 0.4* & 220.7 & 3.3* & - \\
         &              &     &   & Ours  & 26.0 & 25.4 & 27.8 & 0.9* & 3.6* & 135.0 & 0.6* & 155.6 & 9.3* & - \\
\cmidrule(lr){4-15}
         &              &     & 40 & \tool & 41.9 & 41.0 & 47.0 & 6.1* & 7.3* & 178.5 & 0.5* & 220.8 & 3.6* & - \\
         &              &     &   & Ours  & 26.1 & 25.3 & 28.6 & 2.7* & 7.9* & 130.9 & 0.5* & 155.6 & 4.0* & - \\

\bottomrule
\end{tabular}
}
\caption{A detailed list of our experimental results on Fashion-MNIST and CIFAR-10 datasets.}
\label{table:comparison_to_coverd_2}
\end{table*}

\Cref{fig:bound_prop_times} shows the average running time of GPUPoly and \topt for the experiments of~\Cref{fig:bound_prop}.
It shows that our \topt does not introduce significant overhead and is sometimes faster than GPUPoly, for our benchmarks.

\Cref{table:comparison_to_coverd} and \Cref{table:comparison_to_coverd_2} provide a detailed list of our experimental results shown in~\Cref{fig:covev}.
For every dataset, network, $t$, seed (for fixing \tool's random choices), and image index (from the dataset's testset), it shows the analysis time of \tool and \tool+\topt in minutes. 
The notation `–' indicates that an incorrectly classified image (thus not analyzed), `*' denotes that the verification determined that the $\epsilon$-ball is not robust, and `TO' indicates a timeout.

\end{document}